\documentclass[a4paper]{scrartcl}
\usepackage[utf8]{inputenc}
\usepackage{latexsym}
\usepackage{amsmath}
\usepackage{amsthm}
\usepackage{amssymb}
\usepackage{graphicx}
\usepackage{ifthen}
\usepackage{calc}
\usepackage{stmaryrd}
\usepackage{enumitem}
\usepackage{hyperref}
\usepackage{alphalph}

\usepackage{tikz}
\usetikzlibrary{arrows}
\usetikzlibrary{shapes.geometric}
\usetikzlibrary{graphs}
\usetikzlibrary{calc}
\usetikzlibrary{3d}

\usepackage{eucal}

\DeclareFontFamily{OT1}{pzc}{}
\DeclareFontShape{OT1}{pzc}{m}{it}{<-> s * [1.10] pzcmi7t}{}
\DeclareMathAlphabet{\mathpzc}{OT1}{pzc}{m}{it}

\newtheoremstyle{mythm}%
  {}%
  {}%
  {\itshape}%
  {}%
  {\bfseries}%
  {.}%
  {.5em}%
  {\thmname{#1}~\thmnumber{#2}\ifthenelse{\equal{\thmnote{#3}}{}}{}{~(\thmnote{#3})}}%

\newtheoremstyle{mydefn}%
  {}%
  {}%
  {\upshape}%
  {}%
  {\bfseries}%
  {.}%
  {.5em}%
  {\thmname{#1}~\thmnumber{#2}\ifthenelse{\equal{\thmnote{#3}}{}}{}{~(\thmnote{#3})}}%

\newtheoremstyle{myremark}%
  {}%
  {}%
  {\upshape}%
  {}%
  {\itshape}%
  {.}%
  {.5em}%
  {\thmname{#1}~\thmnumber{#2}\ifthenelse{\equal{\thmnote{#3}}{}}{}{~(\thmnote{#3})}}%

\theoremstyle{mythm}
\newtheorem{theorem}{Theorem}[section]

\newtheorem{corollary}[theorem]{Corollary}

\theoremstyle{mydefn}

\newtheorem{example}[theorem]{Example}
\theoremstyle{myremark}

\theoremstyle{mythm}

\newcounter{claimcounter}

\newlist{caselist}{description}{10}
\setlist[caselist]{font=\itshape\mdseries}

\setenumerate[1]{label=(\arabic*)}
\newlist{eroman}{enumerate}{2}
\setlist[eroman,1]{label=(\roman*)}
\setlist[eroman,2]{label=(\alph*)}
\newlist{ealph}{enumerate}{1}
\setlist[ealph]{label=(\Alph*)}

\newcounter{nlistcounter}

\numberwithin{equation}{section}

\usepackage{color}
\definecolor{blau}{RGB}{0,84,159}
\definecolor{hellblau}{RGB}{142,168,229}
\definecolor{petrol}{RGB}{0,97,101}
\definecolor{tuerkis}{RGB}{0,152,161}
\definecolor{gruen}{RGB}{87,171,39}
\definecolor{maigruen}{RGB}{189,205,0}
\definecolor{gelb}{RGB}{255,237,0}
\definecolor{orange}{RGB}{255,128,0}
\definecolor{magenta}{RGB}{227,0,102}
\definecolor{rot}{RGB}{204,7,30}
\definecolor{bordeaux}{RGB}{161,16,53}
\definecolor{violett}{RGB}{97,33,88}
\definecolor{lila}{RGB}{122,111,172}
\definecolor{grey}{gray}{0.7}
\definecolor{mittelblau}{RGB}{0,128,255}

\definecolor{azure}{RGB}{6,154,243}
\definecolor{xindigo}{RGB}{56,2,130}
\definecolor{xorange}{RGB}{255,173,1}
\definecolor{xgreen}{RGB}{21,176,26}

\newcommand{\bigmid}{\mathrel{\big|}}
\newcommand{\Bigmid}{\mathrel{\Big|}}

\newcommand{\angles}[1]{\left\langle#1\right\rangle}

\renewcommand{\hat}{\widehat}

\renewcommand{\vec}[1]{\boldsymbol{#1}}

\renewcommand{\phi}{\varphi}
\renewcommand{\epsilon}{\varepsilon}

\newcommand{\Nat}{{\mathbb N}}
\newcommand{\Real}{{\mathbb R}}

\newcommand{\LC}{\textsf{\upshape C}}

\newcommand{\CC}{{\mathcal C}}

\newcommand{\CF}{{\mathcal F}}
\newcommand{\CG}{{\mathcal G}}

\newcommand{\CP}{{\mathcal P}}

\newcommand{\CT}{{\mathcal T}}

\newcommand{\CX}{{\mathcal X}}

\usepackage[textwidth=2cm]{todonotes}

\usepackage{algorithmic}

\algsetup{linenodelimiter=.}

\usepackage{float}
\newfloat{algorithm}{ht}{alg}
\floatname{algorithm}{Algorithm}

\renewcommand{\hom}{\textsf{\upshape hom}}
\newcommand{\aut}{\textsf{\upshape aut}}
\newcommand{\epi}{\textsf{\upshape epi}}
\newcommand{\emb}{\textsf{\upshape emb}}
\newcommand{\Hom}{\textsf{\upshape Hom}}
\newcommand{\wl}{\textsf{\upshape wl}}
\newcommand{\CTD}{{\mathcal{TD}}}

\newcommand{\wt}{\operatorname{wt}}
\newcommand{\dist}{\operatorname{dist}}

\begin{document}
\title{word2vec, node2vec, graph2vec, X2vec: Towards a Theory of
  Vector Embeddings of Structured Data}
\author{Martin Grohe\\
\normalsize RWTH Aachen University}
\date{}
\maketitle

\begin{abstract}
Vector representations of graphs and relational structures, whether
hand-crafted feature vectors or learned representations, enable us to
apply standard data analysis and machine learning techniques to the
structures. A wide range of methods for generating such embeddings have
been studied
in the machine learning and knowledge representation
literature. However, vector embeddings have received relatively little
attention from a theoretical point of view.

Starting with a survey of embedding techniques that have been used in
practice, in this paper we propose two theoretical approaches that we
see as central for understanding the foundations of vector
embeddings. We draw connections between the various approaches and suggest
directions for future research.
\end{abstract}

\begin{figure}
\centering
\input{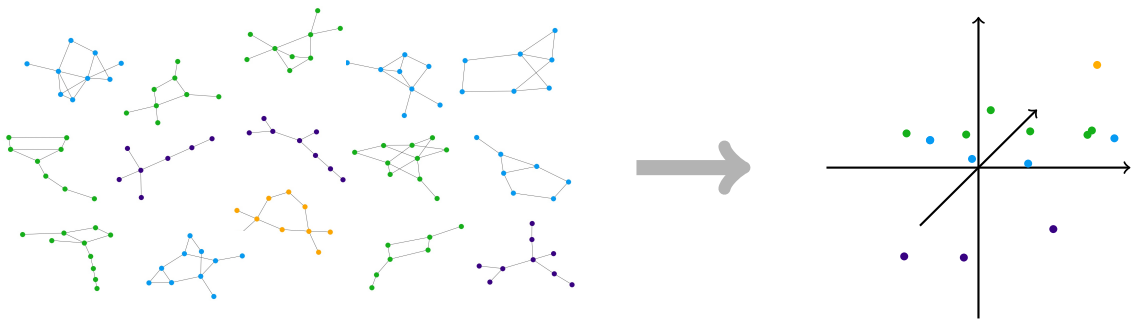}
   \caption{Embedding graphs into a vector space}
   \label{fig:teaser}
 \end{figure}

\section{Introduction}
Typical machine learning algorithms operating on structured data
require representations of the often symbolic data as numerical
vectors. Vector representations of the data range from handcrafted
feature vectors via automatically constructed graph kernels to learned
representations, either computed by dedicated embedding algorithms or
implicitly computed by learning architectures like graph neural
networks. The performance of machine learning methods crucially
depends on the quality of the vector representations. Therefore,
there is a wealth of research proposing a wide range of
vector-embedding methods for various applications.
Most of this research is empirical and often geared towards specific
application areas. Given the importance of the topic, there is
surprisingly little theoretical work on vector embeddings, especially
when it comes to representing structural information that goes beyond
metric information (that is, distances in a graph).

The goal of this paper is to give an overview over the various
embedding techniques for structured data that are used in practice and
to introduce theoretical ideas that can, and to some extent have been
used to understand and analyse them. The research landscape on vector
embeddings is unwieldy, with several communities working largely
independently on related
questions, motivated by different application areas such as social network
analysis, knowledge graphs, chemoinformatics, computational biology,
etc. Therefore, we need to be selective, focussing on common ideas and
connections where we see them.

Vector embeddings can bridge the gap between the ``discrete'' world
of relational data and the ``differentiable'' world of machine
learning and for this reason have a great potential for database
research. Yet relatively little work has been done on embeddings of
relational data beyond the binary relations of knowledge
graphs. Throughout the paper, I will try to point out potential
directions for database related research questions on vector embeddings.

A \emph{vector embedding} for a class $\CX$ of objects is a mapping
$f$ from $\CX$ into some vector space, called the \emph{latent space},
which we usually assume to be a real vector space $\Real^d$ of finite
dimension $d$.  The idea is to define a vector embedding in such a way
that geometric relationships in the latent space reflect semantic
relationships between the objects in $\CX$. Most importantly, we want
similar objects in $\CX$ to be mapped to vectors close to one another
with respect to some standard metric on the latent space (say,
Euclidean).  For example, in an embedding of words of a natural
language we want words with similar meanings, like ``shoe'' and
``boot'', to be mapped to vectors that are close to each other.
Sometimes, we want further-reaching correspondences between properties
of and relations between objects in $\CX$ and the geometry of their
images in latent space. For example, in an embedding $f$ of the
entities of a knowledge base, among them \texttt{Paris},
\texttt{France}, \texttt{Santiago}, \texttt{Chile}, we may want
$\vec t:=f(\texttt{Paris})-f(\texttt{France})$ to be (approximately)
equal to $f(\texttt{Santiago})-f(\texttt{Chile})$, so that the
relation $\texttt{is-capital-of}$ corresponds to the translation by
the vector $\vec t$ in latent space.

A difficulty is that the semantic relationships and similarities
between the objects in $\CX$ can rarely be quantified precisely. They
usually only have an intuitive meaning that, moreover, may be
application dependent. However, this is not necessarily a problem,
because we can learn vector representations in such a way that they
yield good results when we use them to solve machine learning tasks
(so-called \emph{downstream tasks}). This way, we never have to make
the semantic relationships explicit. As a simple example, we may use a
nearest-neighbour based classification algorithm on the vectors our
embedding gives us; if it performs well then the distance between
vectors must be relevant for this classification task. This way,
we can even use vector embeddings, trained to perform well on certain
machine learning tasks, to define semantically meaningful distance
measures on our original objects, that is, to define the distance
$\operatorname{dist}_f(X,Y)$ between objects $X,Y\in\CX$ to be
$\|f(X)-f(Y)\|$. We call $\operatorname{dist}_f$ the distance measure
\emph{induced} by the embedding $f$.

In this paper, the objects $X\in\CX$ we want to embed either are
graphs, possibly labelled or weighted, or more generally relational
structures, or they are nodes of a (presumably large) graph or more
generally elements or tuples appearing in a relational structure. When
we embed entire graphs or structures, we speak of \emph{graph
  embeddings} or \emph{relational structure embeddings}; when we 
embed only nodes or elements we speak of \emph{node embeddings}. These two
types of embeddings are related, but there are clear differences. Most
importantly, in node embeddings there are explicit relations
such as adjacency and derived relations such as distance between the
objects of $\CX$ (the nodes of a graph), whereas in graph embeddings all
relations between objects are implicit or ``semantic'', for example
``having the same number of vertices'' or''having the same girth''
(see Figure~\ref{fig:teaser}).

The key theoretical questions we will ask about vector embeddings of
objects in 
$\CX$ are
the following.
\begin{description}
\item[Expressivity:] Which properties of objects $X\in\CX$ are
  represented by the embedding? What is the meaning of the induced
  distance measure? Are there geometric properties of the latent space
  that represent meaningful relations on $\CX$?
\item[Complexity:] What is the computational cost of
  computing the vector embedding? What are efficient embedding
  algorithms? How can we efficiently retrieve semantic information of
  the embedded data, for example, answer queries?
\end{description}
A third question that relates to both expressivity and complexity is
what dimension to choose for the latent space. In general, we expect a
trade-off between (high) expressivity and (low) dimension, but it may
well be that there is an inherent dimension of the data set. It is an
appealing idea (see, for example, \cite{tensillan00}) to think of ``natural'' data sets appearing in practice as lying on a
low dimensional manifold in high dimensional space. Then we can regard
the dimension of this manifold as the inherent dimension of the data set.

Reasonably well-understood from a theoretical point of view are node
embeddings of graphs that aim to preserve distances between nodes,
that is, embeddings $f:V(G)\to\Real^d$ of the vertex set $V(G)$ of
some graph $G$ such that
$\operatorname{dist}_G(x,y)\approx\|f(x)-f(y)\|$, where
$\operatorname{dist}_G$ is the shortest-path distance in $G$. There is a
substantial theory of such \emph{metric embeddings} (see
\cite{linlonrab95}). In many applications of node embeddings, metric
embeddings are indeed what we need.

However, the metric is only one aspect of the information carried by a
graph or relational structure, and arguably not the most important one
from a database perspective. Moreover, if we consider graph embeddings
rather than node embeddings, there is no metric to start with. In this
paper, we are concerned with \emph{structural} vector embeddings of
graphs, relational structures, and their nodes. Two theoretical ideas that
have been shown to help in understanding and even designing vector
embeddings of structures are the \emph{Weisfeiler-Leman algorithm} and
various concepts in its context, and \emph{homomorphism vectors},
which can be seen as a general framework for defining ``structural''
(as opposed to ``metric'') embeddings. We will see that these
theoretical concepts have a rich theory that connects
them to the embedding techniques used in practice in various ways.

The rest of the paper is organised as follows. Section~\ref{sec:emb}
is a very brief survey of some of the embedding techniques that can be
found in the machine learning and knowledge representation
literature. In Section~\ref{sec:wl}, we introduce the Weisfeiler-Leman
algorithm. This algorithm, originally a graph isomorphism test, turns
out to be an important link between the embedding techniques described
in Section~\ref{sec:emb} and the theory of homomorphism vectors, which
will be discussed in detail in Section~\ref{sec:hom}. Finally,
Section~\ref{sec:sim} is devoted to a discussion of similarity
measures for graphs and structures.

\section{Embedding Techniques}
\label{sec:emb}

In this section, we give a brief and selective overview of embedding
techniques. More thorough recent surveys are \cite{hamyinles17a} (on node
embeddings), \cite{wupanche+19} (on graph neural networks), \cite{wanmaowanguo17} (on knowledge graph embeddings),
and \cite{krijohmor19} (on graph kernels).

\subsection{From Metric Embeddings to Node Embeddings}
\label{sec:node-emb}

Node embeddings can be traced back to the theory of embeddings of
finite metric spaces and dimensionality reduction, which have been
studied in geometry (e.g.\ \cite{bou85,johlin84}) and algorithmic
graph theory (e.g.~\cite{ind01,linlonrab95}). In statistics and data
science, well-known traditional methods of metric embeddings and
dimensionality reduction are multidimensional scaling~\cite{kru64},
Isomap~\cite{tensillan00}, and Laplacian eigenmap~\cite{belniy03}. More
recent related approaches are
\cite{ahmshenar+13,caoluxu15,oucuipei+16,tanquwan15}. The idea is
always to embed the nodes of a graph in such a way that the distance
(or similarity) between vectors approximates the distance (or
similarity) between nodes. Sometimes, this can be viewed as a matrix
factorisation. Suppose we have defined a similarity measure on the nodes of
our graph $G=(V,E)$ that is represented by a similarity matrix
$S\in\Real^{V\times V}$. In the simplest version, we can just take $S$
to be the adjacency matrix of the graph; in the literature this is
sometimes referred to a \emph{first-order proximity}. Another common
choice is $S=(S_{vw})$ with $S_{vw}:=\exp(-c\dist_G(v,w))$, where
$c>0$ is a parameter. We describe our
embedding of $V$ into $\Real^d$ by a matrix $X\in\Real^{V\times d}$ whose rows
$\vec x_v\in\Real^d$ are the images of the nodes. If we measure the
similarity between vectors $\vec x,\vec y$ by their normalised inner
product $\frac{\angles{\vec x,\vec y}}{\|\vec x\|\|\vec y\|}$ (this is
known
as the \emph{cosine similarity}), then
our objective is to find a matrix $X$ with normalised rows that
minimises $\|XX^\top-S\|_F$ with respect to the Frobenius norm
$\|\cdot\|_F$ (or any
other matrix norm, see Section~\ref{sec:sim}). In the basic version
with the Frobenius norm, the problem can be solved
using the singular value decomposition of $S$. 
In a more general
form, we compute a similarity matrix $\hat S\in\Real^{V\times
  V}$ whose $(v,w)$-entry quantifies the similarity between vectors
$\vec x_v,\vec x_w$ and minimise the distance between $S$ and $\hat S$, for example
using stochastic gradient descent. In \cite{hamyinles17a}, this
approach to learning node embeddings is described as an
\emph{encoder-decoder} framework.

\begin{figure}
  \centering
  \input{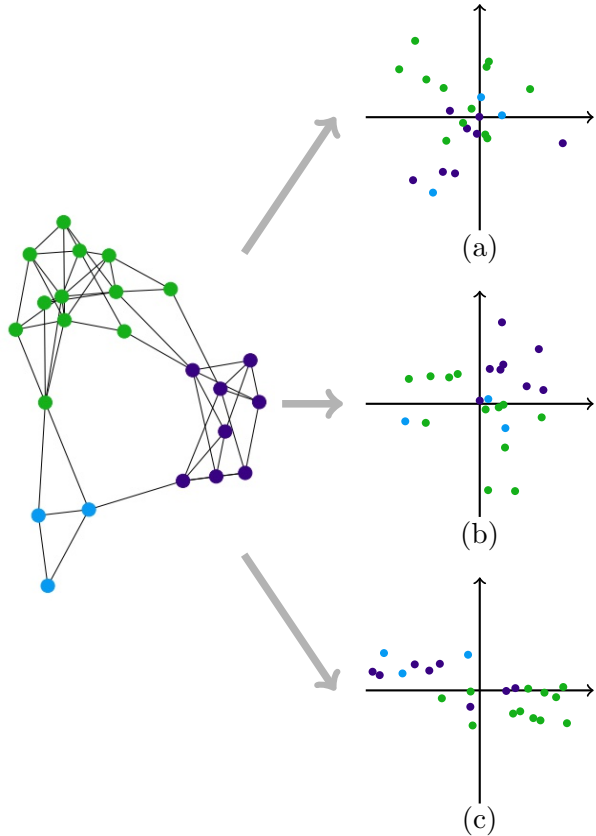}
  \caption{Three node embeddings of a graph using (a) singular value decomposition of
    adjacency matrix, (b) singular value decomposition of the similarity
    matrix with entries $\boldsymbol{S_{vw}=\exp(-2\dist(v,w))}$, (c) \textsc{node2vec} \cite{groles16}}
\label{fig:ne}
\end{figure}

Learned word embeddings and in particular the \textsc{word2vec} algorithm \cite{miksutche+13}
introduced new ideas that had huge impact in natural
language processing. These ideas also inspired new approaches to node
embeddings like \textsc{DeepWalk}~\cite{peralrski14} and
\textsc{node2vec}~\cite{groles16} based on taking short random walks in a graph and
interpreting the sequence of nodes seen on such random walks as if
they were words appearing together in a sentence. These approaches can still be described in the
matrix-similarity (or encoder-decoder) framework: as the similarity
between nodes $v$ and $w$ we take the probability that a fixed-length
random walk starting in $v$ ends in $w$. We can approximate this
probability by sampling random walks. Note that, even in undirected
graphs, this similarity measure is not necessarily symmetric.

From a deep learning perspective, the embedding methods described so
far are all ``shallow'' in that they directly optimise the output
vectors and there are no hidden layers; computing the vector $\vec
x_v$ corresponding to a node $v$ amounts to a table lookup. There are also deep learning methods
for computing node embeddings
(e.g.~\cite{caoluxu16,hamyinles17,wancuizhu16}). Before we discuss such
approaches any further, let us introduce graph neural networks as a
general deep learning framework for graphs that has received a lot of
attention in recent years.

\subsection{Graph Neural Networks}
\label{sec:gnn}
When trying to apply deep learning methods to graphs or relational
structures, we face two immediate difficulties: (1) we want the methods to
scale across different graph sizes, but standard feed-forward neural networks
have a fixed input size; (2) we want the methods to be isomorphism
invariant and not depend on a specific representation of the input
graph. Both points show that it is problematic to just feed the
adjacency matrix (or any other standard representation of a graph)
into a deep neural network in a generic ``end-to-end'' learning
architecture.

\emph{Graph neural networks} (GNNs) are a deep learning framework for
graphs that avoids both of these difficulties, albeit at the price of
limited expressiveness (see Section~\ref{sec:wl}).  Early forms of
GNNs were introduced in \cite{scagortso09,galmic10}; the version we
present here is based on \cite{kipwel17,hamyinles17,phatraphuven17}.
Intuitively, a GNN model can be thought of as a message passing
network over the input graph $G=(V,E)$. Each node $v$ has a state
$\vec x_v\in\Real^d$. Nodes can exchange messages along the edges of
$G$ and update their states. To specify the model, we need to specify
two functions: an \emph{aggregation} function that takes the current
states of the neighbours of a node and aggregates them into a single
vector, and an \emph{update} function that takes the aggregate value
obtained from the neighbours and the current state of the node as
inputs and computes the new state of the node. In a simple form, we
may take the following functions:
\begin{align}
\label{eq:1}
  \textsc{Aggregate}:&&\vec a_v^{(t+1)}\gets\sum_{w\in
                  N(v)}W_{\textsc{agg}}\cdot\vec x_w^{(t)},\\
\label{eq:2}
  \textsc{Update}:&&\vec x_v^{(t+1)}\gets\sigma\left(W_{\textsc{up}}\cdot
                  \begin{pmatrix}
                    \vec x_v^{(t)}\\\vec a_v^{(t+1)}
                  \end{pmatrix}
  \right),
\end{align}
where $W_{\textsc{agg}}\in\Real^{c\times d}$ and
$W_{\textsc{up}}\in\Real^{d\times(c+ d)}$ are learned parameter
matrices and $\sigma$ is a nonlinear ``activation'' function, for
example the ReLU (rectified linear unit) function
$\sigma(x):=\max\{0,x\}$ applied pointwise to a vector. It is important to note that the parameter
matrices $W_{\textsc{agg}}$ and $W_{\textsc{up}}$ do not depend on
the node $v$; they are shared across all nodes of a graph. This
parameter sharing allows
it to use the same GNN model for graphs of arbitrary sizes.

Of course, we can also use more complicated aggregation and update
functions. We only want these functions to be differentiable to be
able to use gradient descent optimisation methods in the training
phase, and we want the aggregation function to be symmetric in its
arguments $\vec x_w^{(t)}$ for $w\in N(v)$ to make sure that the GNN
computes a function that is isomorphism invariant. For example, in
\cite{tonritwolgro19} we use a linear aggregation function and an
update function computed by an LSTM (long short-term memory,
\cite{hocschmi97}), a specific recurrent neural network component that
allows it to ``remember'' relevant information from the sequence
$\vec x_v^{(0)}, \vec x_v^{(1)},\ldots, \vec x_v^{(t)}$.

The computation of such a GNN model starts from a initial
configuration $\big(\vec x_v^{(0)})_{v\in V}$ and proceeds through a
fixed-number $t$ of aggregation- and update-steps, resulting in a
final configuration $\big(\vec x_v^{(t)})_{v\in V}$. Note that this
configuration gives us a node embedding $v\mapsto \vec x_v^{(t)}$ of
the input graph. We can also stack several such GNN layers, each with
its own aggregation and activation function, on top of one another,
using the final configuration of each (but the last) layer as the initial
configuration of the following layer and the final configuration of the last layer as
the node embedding. As initial states, we can take constant vectors
like the all-ones vector for each node, or we can assign a random
initial state to each node. We can also use the initial state to
represent the node labels if the input graph is labelled.

To train a GNN for computing a node-embedding, in principle we can use
any of the loss functions used by the embedding techniques described
in Section~\ref{sec:node-emb}. The reader may wonder what advantage
the complicated GNN architecture has over just optimising the
embedding matrix $X$ (as the methods described in
Section~\ref{sec:node-emb} do). The main advantage is that the GNN method is \emph{inductive}, whereas
the previously described methods are \emph{transductive}. 
This means that
a GNN represents a function that
we can apply to arbitrary graphs, not just to the graph it was
originally trained on. So if the graph changes over time and, for
example, nodes are added, we do not have to re-train the embedding,
but just embed the new nodes using the GNN model we already have,
which is much more efficient. We can even apply the model to an
entirely new graph and still hope it gives us a reasonable embedding.
The most prominent example of an inductive node-embedding
tool based on GNNs is \textsc{GraphSage}~\cite{hamyinles17}.

Let me close this section by remarking that GNNs are used for all
kinds of machine learning tasks on graphs and not only to compute node
embeddings. For example, a GNN based architecture for graph
classification would plug the output of the GNN layer(s) into a standard
feedforward network (possibly consisting only of a single softmax
layer).

\subsection{Knowledge Graph and Relational Structure Embeddings}

Node embeddings of knowledge graphs have also been studied quite
intensely in recent years, remarkably by a community that seems almost
disjoint from that involved in the node embedding techniques described
in Section~\ref{sec:node-emb}. What makes knowledge graphs somewhat
special is that they come with labelled edges (or, equivalently, many
different binary relations) as well as labelled nodes. It is not
completely straightforward to adapt the methods of
Section~\ref{sec:node-emb} to edge- and vertex-labelled
graphs. Another important difference is in the objective function: the
methods of Section~\ref{sec:node-emb} mainly focus on the graph metric
(even though approaches based on random walks like \textsc{node2vec} are flexible and also
incorporate structural criteria). However, shortest-path distance is
less relevant in knowledge graphs.

Rather than focussing on distances, knowledge graph embeddings focus
on establishing a correspondence between the relations of the knowledge
graph and geometric relationships in the latent space. A very
influential algorithm, \textsc{TransE} \cite{borusugar+13} aims to associate a specific
translation of the latent space with each relation. Recall the example of
the introduction, where entities \texttt{Paris},
\texttt{France}, \texttt{Santiago}, \texttt{Chile} were supposed to be
embedded in such a way that
$\vec x_{\texttt{Paris}}-\vec x_{\texttt{France}}\approx
\vec x_{\texttt{Santiago}}-\vec x_{\texttt{Chile}}$, so that the relation $\texttt{is-capital-of}$ corresponds to the translation by
$\vec t:=\vec x_{\texttt{Paris}}-\vec x_{\texttt{France}}$. 

A different algorithm for mapping relations to geometric relationships is
\textsc{Rescal}~\cite{nictrekri11}. Here the idea is to
associate a bilinear form $\beta_R$ with each relation $R$ in such a
way that for all entities $v,w$ it holds that
$\beta_R(\vec x_v,\vec x_w)\approx 1$ if $(v,w)\in R$ and
$\beta_R(\vec x_v,\vec x_w)\approx 0$ if $(v,w)\not\in R$. We can
represent such a bilinear form $\beta_R$ by a matrix $B_R$ such that
$\beta_R(\vec x,\vec y)=\vec x^\top B_R\vec y$. Then the objective is
to minimise, simultaneously for all $R$, the term $\|XB_RX^\top-A_R\|$, where $X$ is the embedding matrix
with rows $\vec x_v$ and $A_R$ is the adjacency matrix
of the relation $R$. Note that this is a multi-relational
version of the matrix-factorisation approach described in
Section~\ref{sec:node-emb}, with the additional twist that we also
need to find the matrix $B_R$ for each relation $R$.

Completing our remarks on knowledge graph embeddings, we mention that
it is fairly straightforward to generalise the GNN based node
embeddings to (vertex- and edge-)labelled graphs and hence to
knowledge graphs \cite{schlikipblo+18}. 

While there is a large body of work on embedding knowledge graphs,
that is, binary relational structures, not much is known about
embedding relations of higher arities. Of course one approach to
embedding relational structures of higher arities is to transform them
into their binary incidence structures (see Section~\ref{sec:beyond}
for a definition) and then embed these using any
of the methods available for binary structures. Currently, I am not
aware of any empirical studies on the practical viability of this
approach.  An alternative
approach \cite{borshm17,borshm19} is based on the idea of treating the
rows of a table, that is, tuples in a relation, like sentences in
natural language and then use word embeddings to embed the entities.

\subsection{Graph Kernels}
\label{sec:kernel}

There are machine learning methods that operate on vector
representations of their input objects, but only use these vector
representations implicitly and never actually access the vectors. All
they need to access is the inner product between two vectors. For
reasons that will become apparent soon, such methods are known as
\emph{kernel methods}. Among them are support vector machines for
classification \cite{corvap95} and principal component analysis as
well as $k$-means clustering for unsupervised learning
\cite{schonsmomul97} (see \cite[Chapter~16]{shaben14} for background).

A \emph{kernel functions} for a set $\CX$ of objects is a binary
function $K:\CX\times\CX\to\Real$ that is symmetric, that is,
$K(x,y)=K(y,x)$ for all $x,y\in\CX$, and \emph{positive semidefinite},
that is, for all $n\ge 1$ and all $x_1,\ldots,x_n\in\CX$ the matrix
$M$ with entries $M_{ij}:=K(x_i,x_j)$ is positive semidefinite. Recall
that a symmetric matrix $M\in\Real^{n\times n}$ is positive
semidefinite if all its eigenvalues are nonnegative, or equivalently,
if $\vec x^TM\vec x\ge 0$ for all $\vec x\in\Real^n$. It can be shown
that a symmetric function $K:\CX\times\CX\to\Real$ is a kernel
function if and only if there is a vector embedding
$f:\CX\to\mathbb H$ of $\CX$ into some Hilbert space $\mathbb H$ such
that $K$ is the mapping induced by the inner product of $\mathbb H$,
that is, $K(x,y)=\angles{f(x),f(y)}$ for all $x,y\in X$ (see
\cite[Lemma~16.2]{shaben14} for a proof). For our purposes, it
suffices to know that a \emph{Hilbert space} is a potentially infinite
dimensional real vector space $\mathbb H$ with a a symmetric bilinear
form $\angles{\cdot,\cdot}:\mathbb H\times \mathbb H\to\Real$
satisfying $\angles{\vec x,\vec x}>0$ for all
$\vec x\in\mathbb H\setminus\{\vec0\}$. A difference between the
embeddings underlying kernels and most other vector embeddings is that
the kernel embeddings usually embed into higher dimensional spaces (even infinite
dimensional Hilbert spaces). Dimension does not play a big role for
kernels, because the embeddings are only used
implicitly. Nevertheless, it can sometimes be more efficient to use
the embeddings underlying kernels explicitly~\cite{krineumor19}.

Kernel methods have dominated machine learning on graphs for a long
time and, despite of the recent successes of GNNs, they are still
competitive. Quoting~\cite{krijohmor19}, \textit{``It remains a current challenge in research
to develop neural techniques for graphs that are able to learn feature representations
that are clearly superior to the fixed feature spaces used by graph
kernels.''}

The first dedicated graph kernels were the random walk graph kernels
\cite{garflawro03,kastsuino03} based on counting walks of all lengths. In some sense, they are
similar to the random walk based node-embedding algorithms described in
Section~\ref{sec:node-emb}. Other graph kernels are based on counting
shortest
paths, trees and cycles, and small subgraph patterns
\cite{borkri05,horgarwro04,ramgar03,shevispet+09}. The important Weisfeiler-Leman kernels~\cite{sheschlee+11}, which we will
describe in Section~\ref{sec:wl}, are based on aggregating local
neighbourhood information. All these kernels are defined for graphs
with discrete labels. The techniques, all essentially based on
counting certain subgraphs, are not directly applicable to graphs
with continuous labels such as edge weights. An adaptation to continuous
labels based on hashing is proposed in \cite{morkrikermut16}.

Let me remark that there are also \emph{node kernels} defined on the
nodes of a graph \cite{konlaf02,neugarker13,smokon03}. They implicitly give a vector
embedding of the nodes. However, compared to the node embedding
techniques discussed before, they only play a minor role. For graph embeddings, we are in the opposite situation: kernels are
the dominant technique. However, there are a few other approaches.

\subsection{Graph Embeddings}
\textsc{Graph2Vec}~\cite{narchaven+17} is a transductive approach to
embedding graphs inspired by \textsc{word2vec} and some of the node
embedding techniques discussed in Section~\ref{sec:node-emb}. As
before, ``transductive'' means that the embedding is computed for a
fixed set of graphs in form of an embedding matrix (or look-up table)
and thus it only yields an embedding for graphs known at training
time. For typical machine learning applications it is unusual to
operate on set of graphs that is fixed in advance, so inductive
embedding approaches are clearly preferable.

GNNs can also be used to embed entire graphs, in the simplest form by
just aggregating the embeddings of the nodes computed by the
GNN. \emph{Graph autoencoders} \cite{kipwel16,panhulon+18} give a way to train such graph
embeddings in an unsupervised manner. A more advanced GNN architecture
for learning graph embedding, 
based on capsule neural networks, is proposed in \cite{xinche19}.

\section{The Weisfeiler-Leman Algorithm}
\label{sec:wl}

We slightly digress from our main theme and introduce the
Weisfeiler-Leman algorithm, a very efficient combinatorial
partitioning algorithm that was originally introduced as a
fingerprinting technique for chemical molecules~\cite{mor65}. The
algorithm plays an important role in the graph isomorphism literature,
both in theory (for example, \cite{bab16,gro17}) and practice, where
it appears as a subroutine in all competitive graph isomorphism tools
(see \cite{mckpi14}). As we will see, the algorithm has interesting
connections with the embedding techniques discussed in the previous
section.
 
\subsection{1-Dimensional Weisfeiler-Leman}
The Weisfeiler-Leman algorithm has a parameter $k$, its \emph{dimension}.
We start by
describing the $1$-dimensional version \emph{$1$-WL}, which is also known as
\emph{colour refinement} or \emph{naive vertex classification}. The
algorithm computes a partition of the nodes of its input
graph. It is convenient to think of the classes of the partition as
colours of the nodes. A colouring (or partition) is \emph{stable} if
any two nodes $v,w$ of the same colour $c$ have the same number of
neighbours of any colour $d$. The algorithm computes a stable
colouring by iteratively refining an initial colouring as described in
Algorithm~\ref{alg:cr}. 
Figure~\ref{fig:colref} shows an example run
of $1$-WL.
The algorithm is very efficient; it can be implemented to %
run in time $O((n+m)\log n)$,
where $n$ is the number of vertices and $m$ the number of edges of the
input graph \cite{carcro82}. Under reasonable assumptions on the
algorithms used, this is best possible~\cite{berbongro17}.

\begin{algorithm}[h]
  \centering
  \fbox{
  \begin{minipage}{\columnwidth-3mm}
  \textsc{1-WL}
  \begin{description}
  \item[Input:] Graph $G$ 
  \item[Initialisation:]All nodes get the same colour.
  \item[Refinement Round:]For all colours $c$ in the current colouring
    and all nodes $v,w$ of colour $c$, the nodes $v$ and $w$ get different
    colours in the new colouring if there is some colour $d$ such that $v$ and $w$ have
    different numbers of neighbours of colour $d$.
  \end{description}
  The refinement is repeated until the colouring is stable, then the
  stable colouring is returned.
  \end{minipage}
}
  \caption{The 1-dimensional WL algorithm}
  \label{alg:cr}
\end{algorithm}

\begin{figure}
  \centering
\begin{tabular}{c@{\hspace{5mm}}c}
  \includegraphics[width=3.2cm]{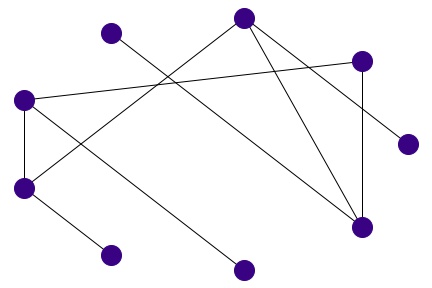}
  &
  \includegraphics[width=3.2cm]{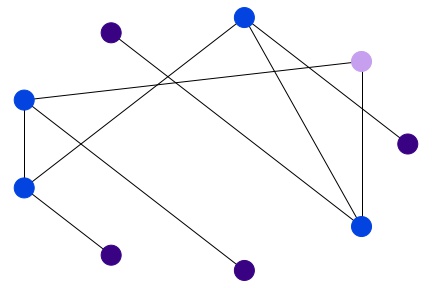}\\
  (a)~initial graph&(b)~colouring after round 1\\[1ex]
  \includegraphics[width=3.2cm]{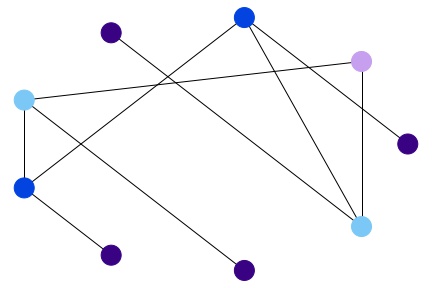}
  &
    \includegraphics[width=3.2cm]{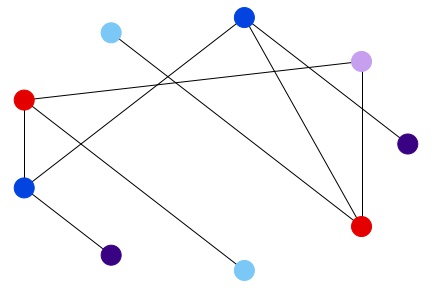}\\
    (c)~colouring after round 2&(d)~stable colouring after round 3
\end{tabular}
  \caption{A run of 1-WL}
  \label{fig:colref}
\end{figure}

To use $1$-WL as an isomorphism test, we note that the colouring computed by the algorithm is \emph{isomorphism
  invariant}, which means that if we run the algorithm on two isomorphic graphs
the resulting coloured graphs will still be isomorphic and in
particular have the same numbers of nodes of each colour. Thus, if we
run the algorithm on two graphs and find that they have distinct
numbers of vertices of some colour, we have produced a certificate of
non-isomorphism. If this is the case, we say that $1$-WL
\emph{distinguishes} the two graphs. Unfortunately, $1$-WL does not distinguish all
non-isomorphic graphs. For example, it does not distinguish a cycle of
length $6$ from the disjoint union of two triangles. 
But, remarkably, $1$-WL does
distinguish \emph{almost all} graphs, in a precise probabilistic sense
\cite{baberdsel80}. 

\subsection{Variants of 1-WL}
The version of $1$-WL we
have formulated is designed for undirected graphs. For directed graphs
it is better to consider in-neighbours and out-neighbours of nodes
separately. $1$-WL can easily be adapted to labelled graphs. If vertex labels are
present, they can be incorporated in the initial colouring: two
vertices get the same initial colour if and only if they have the same
label(s). We can incorporate edge labels in the refinement rounds: two nodes $v$ and $w$ get
different colours in the new colouring if there is some colour $d$ and
some edge label $\lambda$ such that $v$ and $w$ have a different
number of $\lambda$-neighbours of colour $d$. 

However, if the edge labels are real numbers, which we interpret as
edge \emph{weights}, or more generally elements of an arbitrary
commutative monoid, then we can also use the following weighted
version of 1-WL due to \cite{grokermla+14}. Instead of refining by the
number of edges into some colour, we refine by the sum of the edge
weights into that colour. Thus the refinement round of
Algorithm~\ref{alg:cr} is modified as follows: for all colours $c$ in
the current colouring and all nodes $v,w$ of colour $c$, $v$ and $w$
get different colours in the new colouring if there is some colour $d$
such that
\begin{equation}
  \label{eq:3}
  \sum_{x\text{ of colour }d}\alpha(v,x)\quad\neq\sum_{x\text{ of colour
  }d}\alpha(w,x),
\end{equation}
where $\alpha(x,y)$ denotes the weight of the edge from $x$ to $y$,
and we set $\alpha(x,y)=0$ if there is no edge from $x$ to $y$. %
This idea also allows us to define $1$-WL
on matrices: with a matrix $A\in\Real^{m\times n}$ we
associate a weighted bipartite graph with vertex set
$\{v_1,\ldots,v_m,w_1,\ldots,w_n\}$ and edge weights
$\alpha\big(v_i,w_j):=A_{ij}$ and
$\alpha(v_i,v_{i'})=\alpha(w_{j},w_{j'})=0$ and run weighted 1-WL on
this weighted graph with initial colouring that distinguishes the
$v_i$ (rows) from the $w_j$ (columns). An example is shown in
Figure~\ref{fig:matrixWL}.  This matrix-version of WL was applied in
\cite{grokermla+14} to design a dimension reduction techniques that
speeds up the solving of linear programs with many symmetries (or
regularities).

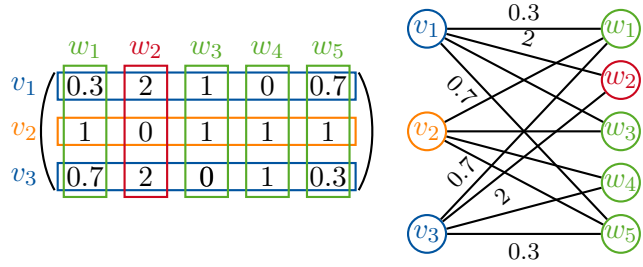
\begin{figure}
  \centering
       \begin{tikzpicture}
        [ 
        thick,
        vertex/.style={draw,circle,inner sep=0pt,minimum
          size=5mm},
        ] 
        %\useasboundingbox (-1,-2.2) rectangle (9,2.2);

        \begin{scope}[x=0.8cm,y=0.6cm]
        \foreach \i/\j/\a in { 1/1/{0.3},1/2/2,1/3/1,1/4/0,1/5/{0.7},2/1/1,2/2/0,2/3/1,2/4/1,2/5/1,3/1/0.7,3/2/2,3/3/0,3/4/1,3/5/0.3,3/3/0}
           \path (\j-1,-\i+2) node {$\a$};
           
           \draw (-0.5,1.3) .. controls (-0.8,0.6) and (-0.8,-0.6) .. (-0.5,-1.3);
           \draw (4.5,1.3) .. controls (4.8,0.6) and (4.8,-0.6)
           .. (4.5,-1.3);

           \path (-1,1) node {\color{blau}$v_1$};
           \draw[blau] (-0.45,0.7) rectangle (4.45,1.3);

           \path (-1,0) node {\color{orange}$v_2$};
           \draw[orange] (-0.45,-0.3) rectangle (4.45,0.3);

           \path (-1,-1) node {\color{blau}$v_3$};
           \draw[blau] (-0.45,-1.3) rectangle (4.45,-0.7);

           \path (0,1.8) node {\color{gruen}$w_1$};
           \draw[gruen] (-0.35,-1.45) rectangle (0.35,1.45);

           \path (1,1.8) node {\color{rot}$w_2$};
           \draw[rot] (0.65,-1.45) rectangle (1.35,1.45);

           \path (2,1.8) node {\color{gruen}$w_3$};
           \draw[gruen] (1.65,-1.45) rectangle (2.35,1.45);

           \path (3,1.8) node {\color{gruen}$w_4$};
           \draw[gruen] (2.65,-1.45) rectangle (3.35,1.45);

           \path (4,1.8) node {\color{gruen}$w_5$};
           \draw[gruen] (3.65,-1.45) rectangle (4.35,1.45);

         \end{scope}
           
        \begin{scope}[xshift=4.5cm,scale=0.85]
        \draw (0,1.6) node[vertex,blau]  (r1) {$v_1$}
              (0,0) node[vertex,orange] (r2) {$v_2$}
              (0,-1.6) node[vertex,blau] (r3) {$v_3$}
              (3,1.6) node[vertex,gruen] (c1) {$w_1$}
              (3,0.8) node[vertex,rot] (c2) {$w_2$}
              (3,0) node[vertex,gruen] (c3) {$w_3$}
              (3,-0.8) node[vertex,gruen] (c4) {$w_4$}
              (3,-1.6) node[vertex,gruen] (c5) {$w_5$}
              ;

       \footnotesize
       \draw (r1) edge node[above] {$0.3$} (c1)
                  edge node[above,sloped] {$2$} (c2)
                  edge (c3)
                  edge node[below,sloped,pos=0.2] {$0.7$} (c5)
       ;
       \draw (r2) edge (c1)
                  edge (c3)
                  edge (c4)
                  edge (c5)
       ;
       \draw (r3) edge node[above,sloped,pos=0.2] {$0.7$} (c1)
                  edge node[below,sloped,pos=0.3] {$2$} (c2)
                  edge node[below] {} (c4)
                  edge node[below] {$0.3$} (c5)
       ;
       \end{scope}
    \end{tikzpicture}

%%% Local Variables: 
%%% mode: latex
%%% TeX-master: "../x2vec"
%%% End: 
  \caption{Stable colouring of a matrix and the corresponding weighted
    bipartite graph computed by matrix WL}
  \label{fig:matrixWL}
\end{figure}

\subsection{Higher-Dimensional WL}
For this paper, the $1$-dimensional version of the  Weisfeiler-Leman
algorithm is the most relevant, but let us briefly describe the higher
dimensional versions. In fact, it is the 2-dimensional version, also
referred to as \emph{classical WL}, that was introduced by Weisfeiler
and Leman \cite{weilem68} in 1968 and gave the algorithm its
name. The \emph{$k$-dimensional
  Weisfeiler-Leman algorithm ($k$-WL)} is based on the same
iterative-refinement idea as $1$-WL. However, instead of vertices, $k$-WL colours $k$-tuples of vertices
of a graph. Initially, each $k$-tuple is ``coloured'' by the
isomorphism type of the subgraph it induces. Then in the refinement
rounds, the colour information is propagated between ``adjacent''
tuples that only differ in one coordinate (details can be found in
\cite{caifurimm92}). 
If implemented using similar ideas
as for $1$-WL, $k$-WL runs in time $O(n^{k+1}\log n)$~\cite{immlan90}.

Higher-dimensional WL is much more powerful than $1$-WL, but Cai, Fürer, and
Immerman~\cite{caifurimm92} proved that for every $k$ there are
non-isomorphic graphs $G_k,H_k$ that are not distinguished by
$k$-WL. These graphs, known as the \emph{CFI graphs}, have
size $O(k)$ and are 3-regular.

\textsc{DeepWL}, a WL-Version of unlimited dimension that can
distinguish the CFI-graphs in polynomial time, was recently introduced
in \cite{groschwewie20b}.

\subsection{Logical and Algebraic Characterisations}
\label{sec:log_alg}
The beauty of the WL algorithm lies in the fact that its
expressiveness has several natural and completely unrelated
characterisations. Of these, we will see two in this section. Later,
we will see two more characterisations in terms of GNNs and homomorphism numbers.

The logic $\LC$
is the extension of first-order logic by counting quantifiers of the
form $\exists^{\ge p}x$ (``there exists at least $p$ elements
$x$''). Every $\LC$-formula is equivalent to a formula of plain
first-order logic. However, here we are interested in fragments of
$\LC$ obtained by restricting the number of
variables of formulas, and the translation from $\LC$ to first-order
logic may increase the number of
variables. For every $k\ge 1$, by $\LC^k$ we denote the fragment of
$\LC$ consisting of all formulas with at most $k$ (free or bound)
variables. The finite variable logics $\LC^k$ play an important role in
finite model theory (see, for example, \cite{grakollib+07}). Cai, Fürer, and Immerman
\cite{caifurimm92} have related these fragments to the WL
algorithm.

\begin{theorem}[\cite{caifurimm92}]\label{theo:cfi}
  Two graphs are $\LC^{k+1}$-equivalent, that is, they satisfy the
  same sentences of the logic $\LC^{k+1}$, if and only if $k$-WL does
  not distinguish the graphs.
\end{theorem}

Let us now turn to an algebraic characterisation of WL. Our starting
point is the observation that two graphs $G,H$ with vertex sets $V,W$
and adjacency matrices $A\in\Real^{V\times
  V},B\in\Real^{W\times W}$ are isomorphic if and
only there is a permutation matrix $X\in\Real^{V\times W}$ such that
$X^\top AX=B$. Recall that a permutation matrix is a $\{0,1\}$-matrix
that has exactly one $1$-entry in each row and in each column. Since
permutation matrices are orthogonal (i.e., they satisfy
$X^\top=X^{-1}$), we can rewrite this as $AX=XB$, which has the
advantage of being linear. This corresponds to the following linear
equations in the variables $X_{vw}$, for $v\in V$ and $w\in W$:
\begin{align}
  \label{eq:4}
  \sum_{v'\in V}A_{vv'}X_{v'w}&=\sum_{w'\in W}X_{vw'}B_{w'w}&\text{for
                                                              all
                                                              }v\in
                                                              V,w\in W.
\intertext{%
We can add equations expressing that the row and column sums of the
matrix $X$ are $1$, which implies that $X$ is a permutation matrix if
the $X_{vw}$ are nonnegative integers.
}
\label{eq:5}
  \sum_{w'\in W}X_{vw'}&=\sum_{v'\in V}X_{v'w}=1&\text{for
                                                              all
                                                              }v\in
                                                              V,w\in W.
\end{align}
Obviously, equations \eqref{eq:4} and \eqref{eq:5} have a nonnegative
integer solution if and only if the graphs $G$ and $H$ are
isomorphic. This does not help much from an algorithmic point of view,
because it is NP-hard to decide if a system of linear equations and
inequalities has an integer solution. But what about nonnegative
rational solutions? We know that we can compute them in polynomial time. A
nonnegative rational solution to \eqref{eq:4} and \eqref{eq:5}, which
can also be seen as a doubly stochastic matrix satisfying $AX=XB$, is
called a \emph{fractional isomorphism} between $G$ and $H$. If such a
fractional isomorphism exists, we say that $G$ and $H$ are \emph{fractionally isomorphic}.
Tinhofer~\cite{tin91} proved the following theorem.

\begin{theorem}[\cite{tin91}]\label{theo:fractional}
  Graphs $G$ and $H$ are fractionally isomorphic if and only if $1$-WL
  does not distinguish $G$ and $H$.
\end{theorem}
 
A corresponding theorem also holds for the weighted and the matrix version of
$1$-WL~\cite{grokermla+14}. Moreover, Atserias and
Maneva~\cite{atsman13} proved a generalisation that relates $k$-WL to
the level-$k$ Sherali-Adams relaxation of the system of equations and
thus yields an algebraic characterisation of $k$-WL
indistinguishability (also see \cite{groott15,mal14} and
\cite{atsoch18,odowriwu+14,bergro15,gragropagpak19} for related
algebraic aspects of WL).

Note that to decide whether two graphs $G,H$ with adjacency matrices
$G,H$ are fractionally isomorphic, we can minimise the convex function
$\|AX-XB\|_F$, where $X$ ranges over the convex set of doubly
stochastic matrices. To minimise this function, we can use
standard gradient descent techniques for convex minimisation.  It was
shown in \cite{kermlagar+14} that, surprisingly, the refinement rounds
of $1$-WL closely correspond to the iterations of the Frank-Wolfe
convex minimisation algorithm. %

\subsection{Weisfeiler-Leman Graph Kernels}
\label{sec:wl-kernels}

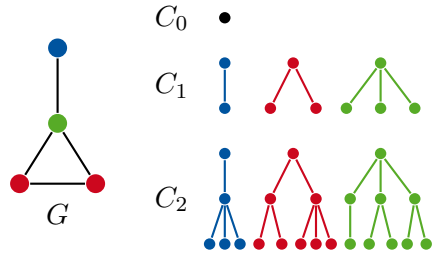
\begin{figure}
  \centering
  \begin{tikzpicture}
  [ 
  thick,
  vertex/.style={fill,circle,inner sep=0pt,minimum
    size=2.5mm},
  tn/.style={fill,circle,inner sep=0pt,minimum
    size=1.5mm},
  ] 
  
  \begin{scope}[xshift=1cm,yshift=-0.2cm]
    \node[vertex,fill=rot] (a) at (-0.5,0) {};
    \node[vertex,fill=rot] (b) at (0.5,0) {};
    \node[vertex,fill=gruen] (c) at (0,0.8) {}; 
    \node[vertex,fill=blau] (d) at (0,1.8) {}; 
    \draw[thick] (a) edge (b) edge (c) (b) edge (c) (c) edge (d);

    \node at (0,-0.4) {$G$};
  \end{scope}

  \begin{scope}[xshift=3cm,yshift=2cm]
    \node[tn] at (0.2,0) {};

    \node at (-0.5,0) {$C_0$};
  \end{scope}

  \begin{scope}[xshift=3cm,yshift=1.4cm]
    \node[tn,blau] (a1) at (0.2,0) {};
    \node[tn,blau] (a2) at (0.2,-0.6) {};
    \draw[thick,blau] (a1) edge (a2);
    
    \node[tn,rot]
    (c1) at (1.1,0) {};
    \node[tn,rot]
    (c2) at (0.8,-0.6) {};
    \node[tn,rot]
    (c3) at (1.4,-0.6) {};
    \draw[thick,rot] (c1) edge (c2) edge (c3);
    
    \node[tn,gruen]
    (b1) at (2.25,0) {};
    \node[tn,gruen]
    (b2) at (1.8,-0.6) {};
    \node[tn,gruen]
    (b3) at (2.25,-0.6) {};
    \node[tn,gruen]
    (b4) at (2.7,-0.6) {};
    \draw[thick,gruen] (b1) edge (b2) edge (b3) edge (b4);

    \node at (-0.5,-0.3) {$C_1$};
  \end{scope}

  \begin{scope}[xshift=3cm,yshift=0.2cm]
       \node[tn,blau]
        (a1) at (0.2,0) {};
        \node[tn,blau]
        (a2) at (0.2,-0.6) {};
       \node[tn,blau]
        (a3) at (0,-1.2) {};
       \node[tn,blau]
        (a4) at (0.2,-1.2) {};
       \node[tn,blau]
        (a5) at (0.4,-1.2) {};
        \draw[thick,blau] (a1) edge (a2) (a2) edge (a3) edge
        (a4) edge (a5);
        
        \node[tn,rot]
        (c1) at (1.1,0) {};
        \node[tn,rot]
        (c2) at (0.8,-0.6) {};
        \node[tn,rot]
        (c3) at (1.4,-0.6) {};
       \node[tn,rot]
        (c4) at (0.65,-1.2) {};
      \node[tn,rot]
        (c5) at (0.95,-1.2) {};
      \node[tn,rot]
        (c6) at (1.2,-1.2) {};
      \node[tn,rot]
        (c7) at (1.4,-1.2) {};
      \node[tn,rot]
        (c8) at (1.6,-1.2) {};

        \draw[thick,rot] (c1) edge (c2) edge (c3) (c2) edge (c4) edge
        (c5) (c3) edge (c6) edge (c7) edge (c8);

        \node[tn,gruen]
        (b1) at (2.25,0) {};
        \node[tn,gruen]
        (b2) at (1.85,-0.6) {};
        \node[tn,gruen]
        (b3) at (2.25,-0.6) {};
        \node[tn,gruen]
        (b4) at (2.7,-0.6) {}; 
        \node[tn,gruen]
        (b5) at (1.85,-1.2) {}; 
        \node[tn,gruen]
        (b6) at (2.1,-1.2) {}; 
       \node[tn,gruen]
        (b7) at (2.4,-1.2) {}; 
       \node[tn,gruen]
        (b8) at (2.6,-1.2) {}; 
       \node[tn,gruen]
        (b9) at (2.85,-1.2) {}; 

        \draw[thick,gruen] (b1) edge (b2) edge (b3) edge (b4) (b2)
        edge (b5) (b3) edge (b6) edge (b7) (b4) edge (b8) edge (b9);

        \node at (-0.5,-0.6) {$C_2$};
    
      \end{scope}
\end{tikzpicture}

%%% Local Variables: 
%%% mode: latex
%%% TeX-master: "../x2vec"
%%% End: 
  \caption{Viewing colours of WL as trees}
  \label{fig:colours}
\end{figure}
 
The WL algorithm collects local structure information and propagates it
along the edges of a graph. We can define very effective graph kernels
based on this local information. For every $i\ge 0$, let $C_i$ be the
set of colours that $1$-WL assigns to the vertices of a graph in the
$i$-th round. Figure~\ref{fig:colours} illustrates that we can
identify the colours in $C_i$ with rooted trees of height $i$. For
every graph $G$ and every colour $c\in C_i$, by
$\wl(c,G)$ we denote the number of vertices that receive
colour $c$ in the $i$th round of $1$-WL.

\begin{example}
  For 
the graph
$G$ shown in Figure~\ref{fig:colours} we have
\[
\wl\left(\tikz[,baseline=-0.5mm,scale=0.8,tn/.style={fill,circle,inner sep=0pt,minimum
    size=1.5mm}]{    \node[tn]
    (c1) at (1.1,0.3) {};
    \node[tn]
    (c2) at (0.8,-0.3) {};
    \node[tn]
    (c3) at (1.4,-0.3) {};
    \draw[thick] (c1) edge (c2) edge (c3);
  },
  G\right)=2,\hspace{1cm}
\wl\left(\tikz[,baseline=-0.5mm,scale=0.8,tn/.style={fill,circle,inner sep=0pt,minimum
    size=1.5mm}]{    \node[tn]
    (c1) at (1.05,0.3) {};
    \node[tn] (c2) at (0.6,-0.3) {};
    \node[tn] (c3) at (0.9,-0.3) {};
    \node[tn] (c4) at (1.2,-0.3) {};
    \node[tn] (c5) at (1.5,-0.3) {};
    \draw[thick] (c1) edge (c2) edge (c3) edge (c4) edge (c5);
  },
  G\right)=0.
\]
\end{example}

For every $t\in\Nat$, the \emph{$t$-round WL-kernel} is
the mapping $K^{(t)}_{\textup{WL}}$ defined by
\[
K^{(t)}_{\textup{WL}}(G,H):=\sum_{i=0}^t\sum_{c\in C_i}\wl(c,G)\cdot\wl(c,H)
\]
for all graphs $G,H$.
It is easy to see that this mapping is symmetric and
positive-semidefinite and thus indeed a kernel mapping; the
corresponding vector embedding maps each graph $G$ to the vector
\[
\Big(\wl(c,G)\Bigmid c\in\bigcup_{i=0}^tC_i\Big).
\]
Note that formally, we are mapping $G$ to an infinite dimensional
vector space, because all the sets $C_i$ for $i\ge 1$ are
infinite. However, for a graph $G$ of order $n$ the vector
as at most $kn+1$ nonzero entries. We can also define a version
$K_{\text{WL}}$ of the
WL-kernel that does not depend on a fixed-number of rounds by letting 
\[
K_{\textup{WL}}(G,H):=\sum_{i\ge 0}\frac{1}{2^i}\sum_{c\in C_i}\wl(c,G)\cdot\wl(c,H).
\]
The WL-kernel was introduced by Shervashidze et
al.~\cite{sheschlee+11} under the name \emph{Weisfeiler-Leman subtree
  kernel}. They also introduce variants such as a
\emph{Weisfeiler-Leman shortest path kernel}. A great advantage the WL
(subtree) kernel has over most of the graph kernels discussed in
Section~\ref{sec:kernel} is its efficiency, while performing at least
as good as other kernels on downstream tasks. Shervashidze et
al.~\cite{sheschlee+11} report that in practice, $t=5$ is a good
number of rounds for the $t$-round WL-kernel.

There are also graph kernels based on higher dimensional WL
algorithm~\cite{morkermut17}.

\subsection{Weisfeiler-Leman and GNNs}
\label{sec:wl-gnn}
Recall that a GNN computes a sequence $(\vec x_v^{(t)})_{v\in V}$, for
$t\ge 0$, of vector embeddings of a graph $G=(V,E)$. In the most
general form, it is recursively defined by 
\[
\vec x_v^{(t+1)}=f_{\textup{UP}}\Big(\vec
x_v^{(t)},f_{\textup{AGG}}\big(\vec x_w\bigmid w\in N(v)\big)\Big),
\]
where the aggregation function $f_{\textup{AGG}}$ is symmetric in its
arguments. It has been observed in several
places~\cite{hamyinles17,morritfey+19,xuhulesjeg19} that this is very
similar to the update process of 1-WL. Indeed, it is easy to see that
if the initial embedding $\vec x_v^{(0)}$ is constant then for any two
vertices $v,w$, if $1$-WL assigns the same colour to $v$ and $w$ then
$\vec x_v^{(t)}=\vec x_w^{(t)}$. This implies that two graphs that
cannot be distinguished by $1$-WL will give the same result for any
GNN applied to them; that is, GNNs are at most as expressive as
$1$-WL. It is shown in \cite{morritfey+19} that a converse of this
holds as well, even if the aggregation and update functions of the GNN
are of a very simple form (like \eqref{eq:1} and \eqref{eq:2} in Section~\ref{sec:gnn}). Based
on the connection between WL and logic, a more refined analysis of the
expressiveness of GNNs was carried out in
\cite{barkosmon+20}. However, the limitations
of the expressiveness only hold if the initial embedding
$\vec x_v^{(0)}$ is constant (or at least constant on all $1$-WL
colour classes). We can increase the expressiveness of GNNs by
assigning random initial vectors $\vec x_v^{(0)}$ to the vertices. The
price we pay for this increased expressiveness is that the output of a
run of the GNN model is no longer isomorphism invariant. However, the
whole randomised process is still isomorphism invariant. More
formally, the random variable that associates an output
$\big(\vec x_v^{(0)})_{v\in V}$ with each graph $G$ is isomorphism
invariant.

A fully invariant way to increase the expressiveness of GNNs is to
build ``higher-dimensional'' GNNs, inspired by the higher-dimensional
WL algorithm. Instead of nodes of the graphs, they operate on
constant sized tuples or sets of vertices. A flexible architecture for
such higher-dimensional GNNs is proposed in \cite{morritfey+19}.

\section{Counting Homomorphisms}
\label{sec:hom}
Most of the graph kernels and also some of the node embedding
techniques are based on counting occurrences of substructures
like walks, cycles, or trees. There are different ways of embedding
substructures into a graph. For example, walks and paths are the
same structures, but we allow repeated vertices in a walk. 
Formally, ``walks'' are homomorphic images of path graphs, whereas
``paths'' are embedded path graphs. 
It turns
that homomorphisms and homomorphic images give us a very robust and
flexible ``basis'' for counting all kinds of substructures
\cite{curdelmar17}.

A homomorphism from a graph $F$ to a graph $G$ is a mapping $h$ from
the nodes of $F$ to the nodes of $G$ such that for all edges $uu'$ of
$F$ the image $h(u)h(u')$ is an edge of $G$. On labelled graphs,
homomorphisms have to preserve vertex and edge labels, and on
directed graphs they have to preserve the edge direction. Of course we can generalise
homomorphisms to arbitrary relational structures, and we remind the
reader of the close connection between homomorphisms and conjunctive
queries. We denote the number of homomorphisms from $F$ to $G$ by
$\hom(F,G)$. 

\begin{example}
  For the graph
$G$ shown in Figure~\ref{fig:colours} we have
\[
\hom\left(\tikz[,baseline=-0.5mm,scale=0.8,tn/.style={fill,circle,inner sep=0pt,minimum
    size=1.5mm}]{    \node[tn]
    (c1) at (1.1,0.3) {};
    \node[tn]
    (c2) at (0.8,-0.3) {};
    \node[tn]
    (c3) at (1.4,-0.3) {};
    \draw[thick] (c1) edge (c2) edge (c3);
  },
  G\right)=18,\hspace{1cm}
\hom\left(\tikz[,baseline=-0.5mm,scale=0.8,tn/.style={fill,circle,inner sep=0pt,minimum
    size=1.5mm}]{    \node[tn]
    (c1) at (1.05,0.3) {};
    \node[tn] (c2) at (0.6,-0.3) {};
    \node[tn] (c3) at (0.9,-0.3) {};
    \node[tn] (c4) at (1.2,-0.3) {};
    \node[tn] (c5) at (1.5,-0.3) {};
    \draw[thick] (c1) edge (c2) edge (c3) edge (c4) edge (c5);
  },
  G\right)=114.
\]
To calculate these numbers, we observe that for the star 
$S_k$ (tree of height $1$ with $k$ leaves) we have
 $\hom(S_k,G)=\sum_{v\in V(G)}\deg_G(v)^k$.
\end{example}

For every class $\CF$ of graphs, the homomorphism counts $\hom(F,G)$ give a
graph embedding $\Hom_{\CF}$ defined by
\[
  \Hom_{\CF}(G):=\big(\hom(F,G)\bigmid F\in\CF\big)
\]
for all graphs $G$.  If $\CF$ is infinite, the latent space
$\Real^{\CF}$ of the embedding $\Hom_{\CF}$ is an infinite dimensional
vector space.  By suitably scaling the infinite series involved, we can define an inner product on a subspace
$\mathbb H_{\CF}$ of $\Real^{\CF}$ that includes the range of
$\Hom_{\CF}$. This
also gives us a graph kernel. One way of making this precise is as
follows. For every $k$, we let $\CF_k$ be the set of all $F\in\CF$ of
order $|F|:=|V(F)|=k$. Then we let
\begin{equation}
  \label{eq:30}
   K_{\CF}(G,H):=\sum_{k=1}^\infty\frac{1}{|\CF_k|}\sum_{F\in\CF_k}\frac{1}{k^k}\hom(F,G)\cdot\hom(F,H).
\end{equation}
There are various other ways of doing this, for example, rather than
looking at the sum over all $F\in\CF^k$ we may look at the maximum. In
practice, one will simply cut off of the infinite series and only consider
a finite subset of $\CF$. A problem with using homomorphism vectors
as graph embeddings is that the homomorphism numbers quickly get
tremendously large. In practice, we take logarithms of theses numbers,
possibly scaled by the size of the graphs from $\CF$. So, a
practically reasonable
graph embedding based on homomorphism vectors would take a finite
class $\CF$ of graphs and map each $G$ to the vector
\[
  \Big(\frac{1}{|F|}\log\big(\hom(F,G)\big)\Bigmid F\in\CF\Big).
\]
Initial experiments show that this graph embedding performs very
well on downstream classification tasks even if we take $\CF$ to be a
small class (of size $20$) of graphs consisting of binary trees and
cycles. This is a good indication that homomorphism vectors extract
relevant features from a graph. Note that the size of the class $\CF$
is the dimension of the feature space.

Apart from these practical considerations, homomorphisms vectors have
a beautiful theory that links them to various natural notions of
similarity between structures, including indistinguishability by the
Weisfeiler-Leman algorithm. 

\subsection{Homomorphism Indistinguishability}
\label{sec:hi}

Two graphs $G$ and $H$ are \emph{homomorphism-indistinguishable} over a
class $\CF$ of a graphs if $\Hom_{\CF}(G)=\Hom_{\CF}(H)$.  Lov\'asz
proved that homomorphism indistinguishability over the class $\CG$ of
all graphs corresponds to isomorphism.

\begin{theorem}[\cite{lov67}]\label{theo:lov}
  For all graphs $G$ and $H$,
  \[
    \Hom_{\CG}(G)=\Hom_{\CG}(H)\iff G\text{ and }H\text{ are
      isomorphic.}
  \]
\end{theorem}

\begin{proof}
 The backward
  direction is trivial. For the forward direction, suppose that
  $\Hom_{\CG}(G)=\Hom_{\CG}(H)$, that is, for all graphs $F$ it holds
  that $\hom(F,G)=\hom(F,H)$.

We can decompose
every homomorphism $h:F\to F'$ as $h=f\circ g$ such that for some graph
$F''$:
\begin{itemize}
\item $g:F\to F''$ is an \emph{epimorphism}, that is, a homomorphism
  such that for every $v\in
V(F'')$ there is a $u\in V(F)$ with $g(u)=v$ and for every
$vv'\in E(F'')$ there is a $uu'\in E(F)$ with
$g(u)=v$ and $g(u')=v'$;
\item $f:F''\to F'$ is an \emph{embedding} (or \emph{monomorphism}),
  that is, a homomorphism such that $f(u)\neq f(u')$ for all $u\neq u'$.
\end{itemize}
Note that the graph $F''$ is isomorphic to the image $h(F)$ and thus
unique up to isomorphism. Moreover, there are precisely 
\[
\aut(F'')=\text{number of automorphisms of }F''
\] 
isomorphisms from $F''$ to
$h(F)$. This means that there are $\aut(F'')$ pairs $(f,g)$ such
that $h=f\circ g$ and $g:F\to F''$ is an epimorphism and $f:F''\to F'$
is an embedding. Thus we can write
\begin{equation}
  \label{eq:10}
  \hom(F,F')=\sum_{F''}\frac{1}{\aut(F'')}\cdot\epi(F,F'')\cdot\emb(F'',F'),
\end{equation}
where $\epi(F,F'')$ is the number of epimorphisms from $F$ onto $F''$,
$\emb(F'',F')$ is the number of embeddings of $F''$ into $F'$, and the
sum ranges over all isomorphism types of graphs $F''$.
Furthermore, as $\epi(F,F'')=0$ if $|F''|>|F|$, we can
restrict the sum to $F''$ of order $|F''|\le|F|$.

Let $F_1,\ldots,F_m$ be an enumeration of all graphs of order at
most $n:=\max\{|G|,|H|\}$ such that each graph of order at most $n$  is isomorphic to exactly
one graph in this list and that $i\le j$ implies $|F_i|< |F_j|$ or
$|F_i|= |F_j|$ and $\|F_i\|:=|E(F_i)|\le\|F_j\|$. Let $\mathit{HOM}$ be the
matrix with entries $\mathit{HOM}_{ij}:=\hom(F_i,F_j)$, $P$ the matrix with
entries $P_{ij}:=\epi(F_i,F_j)$, $M$ the matrix with
entries $M_{ij}:=\emb(F_i,F_j)$, and $D$ the diagonal matrix with
entries $D_{ii}:=\frac{1}{\aut(F_i)}$. Then \eqref{eq:10} yields the
following matrix equation;
\begin{equation}
  \label{eq:6}
  \mathit{HOM}=P\cdot D\cdot M.
\end{equation}
The crucial observation is that $P$ is an lower triangular matrix,
because $\epi(F_i,F_j)>0$ implies $|F_i|\ge |F_j|$ and
$\|F_i\|\ge\|F_j\|$. Moreover, $P$ has positive diagonal entries,
because $\epi(F_i,F_i)\ge 1$. Similarly, $M$ is an upper triangular
matrix with positive diagonal entries. Thus $P$ and $M$ are
invertible. The diagonal
matrix $D$ is invertible as well, because $D_{ii}=\aut(F_i)\ge
1$. Thus the matrix $\mathit{HOM}$ is invertible.

As $|G|,|H|\le n$ there are graphs $F_i,F_j$ such that $G$ is
isomorphic to $F_i$ and $H$ is isomorphic to $F_j$. Since
$\hom(F_k,G)=\hom(F_k,H)$ for all $k$ by the assumption of the lemma,
the $i$th and $j$th column of $\mathit{HOM}$ are identical. As $\mathit{HOM}$ is
invertible, this implies that $i=j$ and thus that $G$ and $H$ are
isomorphic.
\end{proof}

The theorem can be seen as a the starting point for the theory of graph
limits \cite{borchalov+06,lov12,lovsze06}. The graph embedding
$\Hom_{\CG}$ maps
graphs into an infinite dimensional real vector space, which can be
turned into a Hilbert space by defining a suitable inner product. This
transformation enables us to analyse graphs with methods of linear
algebra and functional analysis and, for example, to consider convergent
sequences of graphs and their limits, called \emph{graphons} (see
\cite{lov12}).

However, not only the ``full'' homomorphism vector $\Hom_{\CG}(G)$ of
a graph $G$, but also its projections $\Hom_{\CF}(G)$ to natural
classes $\CF$ capture very interesting information about $G$. A first
result worth mentioning is the following. This result is well-known,
though usually phrased differently. Two graphs are \emph{co-spectral}
if their adjacency matrices have the same eigenvalues with the same
multiplicities. Figure~\ref{fig:cospectral} shows two graph that are
co-spectral, but not isomorphic.

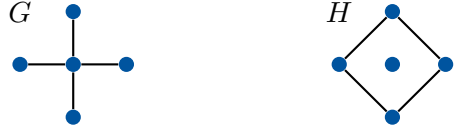
\begin{figure}
  \centering
  \begin{tikzpicture}[
  vertex/.style = {fill=blau,circle,inner sep=0pt, minimum height
    = 2mm},
  scale=0.7
  ]
  \begin{scope}
    \node[vertex] (a)   at (0   ,0  ) {};
    \node[vertex] (b1)  at (1   ,0) {};
    \node[vertex] (b2)  at (0   ,1) {};
    \node[vertex] (b3)  at (-1   ,0) {};
    \node[vertex] (b4)  at (0   ,-1) {};
    
    \draw[thick] (a) edge (b1) edge (b2) edge (b3) edge (b4);

    \path (-1,1) node {$G$};
  \end{scope}
  \begin{scope}[xshift=6cm]
    \node[vertex] (a)   at (0   ,0  ) {};
    \node[vertex] (b1)  at (1   ,0) {};
    \node[vertex] (b2)  at (0   ,1) {};
    \node[vertex] (b3)  at (-1   ,0) {};
    \node[vertex] (b4)  at (0   ,-1) {};
    
    \draw[thick] (b1) edge (b2) (b2) edge (b3) (b3) edge (b4) (b4)
    edge (b1);

    \path (-1,1) node {$H$};

  \end{scope}
\end{tikzpicture}

%%% Local Variables: 
%%% mode: latex
%%% TeX-master: "../x2vec"
%%% End: 
  \caption{Co-spectral graphs}
  \label{fig:cospectral}
\end{figure}

\begin{theorem}[Folklore]\label{theo:cycle}
  For all graphs $G$ and $H$, 
  \[
    \Hom_{\CC}(G)=\Hom_{\CC}(H)\iff G\text{ and }H\text{ are
     co-spectral.}
  \]
  Here $\CC$ denotes the class of all cycles.
\end{theorem}

\begin{proof}[Proof sketch]
  Observe that for the cycle $C_k$ of length $k$ we have
  $\hom(C_k,G):=\operatorname{trace}(A^k)$, where $A$ is the adjacency matrix of
  $G$. It is well-known that the trace of a symmetric real matrix is
  the sum of its eigenvalues and that the eigenvalues of $A^k$ are the
  $k$th powers of the eigenvalues of $A$. This
  immediately implies the backward direction. By a simple
  linear-algebraic argument, it also implies the forward direction.
\end{proof}

Dvor\'ak~\cite{dvo10} proved that homomorphism counts of trees, and more generally, graphs of bounded tree
width link homomorphism vectors to the Weisfeiler-Leman algorithm.

\begin{theorem}[\cite{dvo10}]\label{theo:treewidth}
  For all graphs $G$ and $H$ and all $k\ge 1$,
  \[
    \Hom_{\CT_k}(G)=\Hom_{\CT_k}(H)\iff \text{$k$-WL does not
      distinguish $G$ and $H$.}
  \]
  Here $\CT_k$ denotes the class of all graphs of tree width at most $k$.
\end{theorem}

To prove the theorem, it suffices to consider connected graphs in
$\CT_k$, because for a graph $F$ with connected components
$F_1,\ldots,F_m$ we have $\hom(F,G)=\prod_{i=1}^m\hom(F_i,G)$. The
connected graphs of tree width $1$ are the trees.
We sketch a proof of the theorem for trees in Section~\ref{sec:hom-ne}.

In combination with Theorem~\ref{theo:fractional},
Theorem~\ref{theo:treewidth} implies
the following.

\begin{corollary}
  For all graphs $G$ and $H$,
    \[
    \Hom_{\CT}(G)=\Hom_{\CT}(H)\iff \parbox[t]{4cm}{$G$ and $H$ are
      fractionally isomorphic, that is, equations
\eqref{eq:4} and \eqref{eq:5} have a nonnegative rational solution.}
\]
 Here $\CT$ denotes the class of all trees.
\end{corollary}

Remarkably, for paths we obtain a similar characterisation that
involves the same equations, but drops the nonnegativity constraint.

\begin{theorem}[\cite{delgrorat18}]\label{theo:path}
    For all graphs $G$ and $H$,
    \[
    \Hom_{\CP}(G)=\Hom_{\CP}(H)\iff \parbox[t]{4cm}{equations
\eqref{eq:4} and \eqref{eq:5} have a\\ rational solution.}
\]
 Here $\CP$ denotes the class of all paths.
\end{theorem}

While the proof of Theorem~\ref{theo:treewidth} relies on techniques
similar to the proof of Theorem~\ref{theo:lov}, the proof of
Theorem~\ref{theo:cycle} is based on spectral techniques similar to
the proof of Theorem~\ref{theo:cycle}.

\begin{example}
  For the co-spectral graphs $G,H$ shown in
  Figure~\ref{fig:cospectral} we have
  \[
    \hom\left(\tikz[,baseline=-0.5mm,scale=0.8,tn/.style={fill,circle,inner sep=0pt,minimum
    size=1.5mm}]{    \node[tn]
    (c1) at (1.1,0.3) {};
    \node[tn]
    (c2) at (0.8,-0.3) {};
    \node[tn]
    (c3) at (1.4,-0.3) {};
    \draw[thick] (c1) edge (c2) edge (c3);
  },
  G\right)=20,\hspace{1cm}
    \hom\left(\tikz[,baseline=-0.5mm,scale=0.8,tn/.style={fill,circle,inner sep=0pt,minimum
    size=1.5mm}]{    \node[tn]
    (c1) at (1.1,0.3) {};
    \node[tn]
    (c2) at (0.8,-0.3) {};
    \node[tn]
    (c3) at (1.4,-0.3) {};
    \draw[thick] (c1) edge (c2) edge (c3);
  },
  H\right)=16.
\]
Thus $\Hom_{\CP}(G)\neq\Hom_{\CP}(H)$.
\end{example}

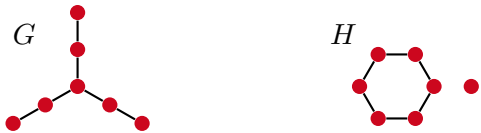
\begin{figure}
  \centering
  \begin{tikzpicture}[
  vertex/.style = {fill=rot,circle,inner sep=0pt, minimum height
    = 2mm},
  scale=0.7
  ]
  \begin{scope}
    \node[vertex] (n)   at (0   :0  ) {};
    \node[vertex] (n1)  at (330   :0.7) {};
    \node[vertex] (n1b) at (330   :1.4) {};
    \node[vertex] (n2)  at (90 :0.7) {};
    \node[vertex] (n2b) at (90 :1.4) {};
    \node[vertex] (n3)  at (210:0.7) {};
    \node[vertex] (n3b) at (210:1.4) {};
    \draw[thick] (n) -- (n1) -- (n1b);
    \draw[thick] (n) -- (n2) -- (n2b);
    \draw[thick] (n) -- (n3) -- (n3b);
    
    \path (-1,1) node {$G$};
  \end{scope}

  \begin{scope}[xshift=6cm]
    \node[vertex] (n1)  at (0   :0.7) {};
    \node[vertex] (n2)  at (60  :0.7) {};
    \node[vertex] (n3)  at (120 :0.7) {};
    \node[vertex] (n4)  at (180 :0.7) {};
    \node[vertex] (n5)  at (240 :0.7) {};
    \node[vertex] (n6)  at (300 :0.7) {};
    \node[vertex] (n)   at (0:1.4) {};
    \draw[thick] (n1) edge (n2) edge (n6) (n3) edge (n2) edge (n4)
    (n5) edge (n6) edge (n4);    

    \path (-1,1) node {$H$};

  \end{scope}
\end{tikzpicture}

%%% Local Variables: 
%%% mode: latex
%%% TeX-master: "../x2vec"
%%% End: 
  \caption{Two graphs that are homorphism-indistinguishable over the
    class of paths}
  \label{fig:copath}
\end{figure}

\begin{example}
  Figure~\ref{fig:copath} shows graphs $G,H$ with
  \[
    \Hom_{\CP}(G)=\Hom_{\CP}(H).
  \] 
  Obviously, 1-WL distinguishes the two
  graphs. Thus $\Hom_{\CT}(G)\neq\Hom_{\CT}(H)$. It can also be
  checked that the graphs are not co-spectral. Hence $\Hom_{\CC}(G)\neq\Hom_{\CC}(H)$
\end{example}

Combined with Theorem~\ref{theo:cfi}, Theorem~\ref{theo:treewidth}
implies the following correspondence between homomorphism counts of
graphs of bounded tree width and the finite variable fragments of the
counting logic $\LC$ introduced in Section~\ref{sec:log_alg}.

\begin{corollary}
  For all graphs $G$ and $H$ and all $k\ge 1$,
  \[
    \Hom_{\CT_k}(G)=\Hom_{\CT_k}(H)\iff \text{$G$ and $H$ are
      $\LC^{k+1}$-equivalent.}
  \]
\end{corollary}

This is interesting, because it shows that the homomorphism vector
$\Hom_{\CT_k}(G)$ gives us all the information necessary to answer
queries expressed in the logic $\LC^{k+1}$. Unfortunately, the result does not
tell us how to answer $\LC^{k+1}$-queries algorithmically if we have access to
the vector $\Hom_{\CT_k}(G)$. To make the question precise, suppose we
have oracle access that allows us to obtain, for every graph $F\in\CT_k$,
the entry $\hom(F,G)$ of the homomorphism vector. Is it possible to
answer a $\LC^{k+1}$-query in polynomial time (either with respect
to data complexity or combined complexity)?

Arguably, from a logical perspective it is even more natural to
restrict the \emph{quantifier rank} (maximum number of nested
quantifiers) in a formula rather than the number of variables. Let
$\LC_k$ be the fragment of $\LC$ consisting of all formulas of
quantifier rank at most $k$. We obtain the following characterisation
of $\LC_k$-equivalence in terms of homomorphism vectors over the class
of graphs of \emph{tree depth} at most $k$. Tree depth, introduced by
Ne\v set\v ril and Ossona de Mendez \cite{nesoss06}, is another
structural graph parameter that has received a lot of attention in
recent years
(e.g.~\cite{bantan16,buldaw14,cheflu18,elbjaktan12,elbgrotan16}).

\begin{theorem}[\cite{gro20b}]\label{theo:treedepth}
    For all graphs $G$ and $H$ and all $k\ge 1$,
  \[
    \Hom_{\CTD_k}(G)=\Hom_{\CTD_k}(H)\iff \text{$G$ and $H$ are
      $\LC_{k}$-.equivalent.}
  \]
  Here $\CTD_k$ denotes the class of all graphs of tree depth at most $k$.
\end{theorem}

Another very remarkable result, due to Man\v cinska and
Roberson~\cite{manrob19}, states that two graphs are
homomorphism-indistinguish\-able over the class of all planar graphs if
and only if they are \emph{quantum isomorphic}. Quantum isomorphism,
introduced in \cite{atsmanrob+19}, is a complicated notion
that is based on similar systems of equations as \eqref{eq:4} and
\eqref{eq:5} and their generalisation characterising
higher-dimensional Weisfeiler-Leman indistinguishability,
but with non-commutative variables ranging over the elements of
a $C^*$-algebra.

\subsection{Beyond Undirected Graphs}
\label{sec:beyond}

So far, we have only considered homomorphism indistinguishability on
undirected graphs. A few results are known for directed
graphs. In particular, Theorem~\ref{theo:lov} directly extends to directed
graphs. Actually, we have the following stronger
result, also due to Lov\'asz \cite{lov71} (also see
\cite{bok19}).

\begin{theorem}[\cite{lov71}]
  For all directed graphs $G$ and $H$,
  \[
    \Hom_{\mathcal{DA}}(G)=\Hom_{\mathcal{DA}}(G)\iff G\text{ and }H\text{ are
      isomorphic.}
  \]
  Here $\mathcal{DA}$ denote the class of all directed acyclic graphs.
\end{theorem}

It is straightforward to extend Theorem~\ref{theo:lov},
Theorem~\ref{theo:treewidth} (for the natural generalisation of the WL
algorithms to relational structures), and Theorem~\ref{theo:treedepth} to
arbitrary relational structures. This is very useful for binary
relational structures such as knowledge graphs. But for relations of
higher arity one may consider another version based on the incidence
graph of a structure. 

Let $\sigma=\{R_1,\ldots,R_m\}$ be a relational vocabulary, where
$R_i$ is a $k_i$-ary relation symbol. Let $k$ be the maximum of the
$k_i$. We let $\sigma_I:=\{E_1,\ldots,E_k,P_1,\ldots,P_m\}$, where the
$E_i$ are binary and the $P_j$ are unary relation symbols. With every
$\sigma$-structure $A=(V(A),R_1(A),\ldots,R_m(A))$ we associates a
$\sigma_I$-structure $A_I$ called the \emph{incidence structure} of
$A$ as follows:
\begin{itemize}
\item the universe of $A_I$ is
  \[
    V(A_I):=V(A)\cup\bigcup_{i=1}^m\big\{(R_i,v_1,\ldots,v_{k_i}\bigmid
    (v_1,\ldots,v_{k_i})\in R_i(A)\big\};
  \]
\item for $1\le j\le k$, the relation $E_j(A_I)$ consists of all pairs
  \[
    \big(v_j,(R_i,v_1,\ldots,v_{k_i})\big)
  \]
  for $(R_i,v_1,\ldots,v_{k_i})\in V(A_I)$ with $k_i\ge j$;
\item
  for $1\le i\le m$, the relation $P_i$ consist of all $(R_i,v_1,\ldots,v_{k_i})\in V(A_I)$.
\end{itemize}
With this encoding of general structures as binary incidence
structures we obtain the following corollary.

\begin{corollary}
  For all $\sigma$-structures $A$ and $B$, the following are
  equivalent.
  \begin{enumerate}
  \item $\Hom_{\CT(\sigma_I)}(A_I)=\Hom_{\CT(\sigma_I)}(B_I)$, where
    $\CT(\sigma_I)$ denotes the class of all $\sigma_I$-structures
    whose underlying (Gaifman) graph is a tree;
  \item $A_I$ and $B_I$ are not distinguished by $1$-WL;
  \item $A_I$ and $B_I$ are $\LC^2$-equivalent.
  \end{enumerate}
\end{corollary}

Böker~\cite{bok19} gave a generalisation of
Theorem~\ref{theo:treewidth} to hypergraphs that is also based on
incidence graphs.

In a different direction, we can generalise the results to weighted
graphs. Let us consider undirected graphs with real-valued edge
weights. We can also view them as symmetric matrices over the
reals. Recall that we denote the weight of an edge $uv$ by
$\alpha(u,v)$ and that weighted 1-WL refines by sums of edge
weights (instead of numbers of edges). Let $F$ be an unweighted graph and
$G$ a weighted graph. For every mapping $h:V(F)\to V(G)$ we let
\[
  \wt(h):=\prod_{uu'\in E(F)}\alpha(h(u),h(u')).
\]
As $\alpha(v,v')=0$ if and only if $vv'\not\in E(G)$, we have 
$\wt(h)\neq0$ if and only if $h$ is a homomorphism from $F$
to $G$; the weight of this homomorphism is the product of the weights of
the edges in its image.
We let
\[
  \hom(F,G):=\sum_{h:V(F)\to V(G)}\wt(h).
\]
In statistical physics, such sum-product functions are known as
\emph{partition functions}. For a class $\CF$ of graphs, we let
\[
\Hom_{\CF}(G):=\big(\hom(F,G)\bigmid F\in\CF\big).
\]

\begin{theorem}[\cite{bulgrorat20+}]
    For all weighted graphs $G$ and $H$, the following are
  equivalent.
  \begin{enumerate}
  \item $\Hom_{\CT}(G)=\Hom_{\CT}(H)$ (recall that $\CT$ denotes the
    class of all trees);
  \item $G$ and $H$ are not distinguished by weighted $1$-WL;
  \item equations
\eqref{eq:4} and \eqref{eq:5} have a nonnegative rational solution.
  \end{enumerate}
\end{theorem}

\subsection{Complexity}

In general, counting the number of homomorphisms from a graph $F$ to a
graph $G$ is a $\#\textsf{P}$-hard problem. Dalmau and Jonsson~\cite{daljon04}
proved that, under the reasonable complexity theoretic assumption
$\#\textsf{W[1]}\neq\textsf{FPT}$ from parameterised complexity
theory, for all classes $\CF$ of graphs, computing $\hom(F,G)$ given a graph
$F\in\CF$ and an arbitrary graph $G$, is in polynomial time if and
only if $\CF$ has bounded tree width. This makes
Theorem~\ref{theo:treewidth} even more interesting, because the
entries of a homomorphism vector $\Hom_{\CF}(G)$ are computable in
polynomial time precisely for bounded tree width classes.

However, the computational problem we are facing is not to compute
individual entries of a homomorphism vectors, but to decide if two
graphs have the same homomorphism vector, that is, if they are homomorphism
indistinguishable. The characterisation theorems of
Section~\ref{sec:hi} imply that homomorphism
indistinguishability is polynomial-time decidable over the classes $\CP$
of paths, $\CC$ of cycles, $\CT$ of trees, $\CT_k$ of graphs of tree
width at most $k$, $\CTD_k$ of tree depth at most $k$. Moreover,
homomorphism indistinguishability over the class of $\CG$ of all graphs
is decidable in quasi-polynomial time by Babai's \cite{bab16}
celebrated result that graph isomorphism is decidable in
quasi-polynomial time. It was proved in \cite{bokchegrorat19} that
homomorphism indistinguishability over the class of complete graphs is
complete for the complexity class $\textsf{C}_=\textsf{P}$, which
implies that it is \textsf{co-NP} hard, and that there is a polynomial
time decidable class $\CF$ of graphs of bounded tree width such that
homomorphism indistinguishability over $\CF$ is undecidable. Quite
surprisingly, the fact that quantum isomorphism is undecidable
\cite{atsmanrob+19} implies that homomorphism distinguishability
over the class of planar graphs is undecidable.

\subsection{Homomorphism Node Embeddings}
\label{sec:hom-ne}

So far, we have defined graph and structure embeddings based on
homomorphism vectors. But we can also use homomorphism vectors to define
node embeddings. A \emph{rooted graph} is a pair $(G,v)$ where $G$ is a
graph and $v\in
V(G)$. For two rooted graphs $(F,u)$ and $(G,v)$, by $\hom(F,G;u\mapsto
v)$ we denote the number of homomorphism $h$ from $F$ to $G$ with
$h(u)=v$. For a class $\CF^*$ of rooted graphs and a rooted graph
$(G,v)$, we let
\[
\Hom_{\CF^*}(G,v):=\big(\hom(F,G;u\mapsto v)\bigmid (F,u)\in\CF^*\big).
\]
If we keep the graph $G$ fixed, this gives us an embedding of the
nodes of $G$
into an infinite dimensional vector space. Note that in the terminology
of Section~\ref{sec:node-emb}, this embedding is ``inductive'' and
not ``transductive'', because it is not tied to a fixed graph. (Nevertheless,
the term ``inductive'' is not fitting very well here, because the
embedding is not learned.) In the same way we defined graph kernels
based on homomorphism vectors of graphs, we can now define node kernels.

It is straightforward to generalise Theorem~\ref{theo:lov} to rooted
graphs, showing that for all rooted graphs $(G,v)$ and $(H,w)$ it
holds that
\[
\Hom_{\CG^*}(G,v)=\Hom_{\CG^*}(H,w)\iff\parbox[t]{3.5cm}{there is an
  isopmorphism $f$ from $G$ to $H$ with $f(v)=w$.}
\]
Here $\CG^*$ denote the class of all rooted graphs. Maybe the easiest
way to prove this is by a reduction to node-labelled graphs.

Another key result of Section~\ref{sec:hi} that can be adapted to the
node setting is Theorem~\ref{theo:treewidth}. We only state the version
for trees.

\begin{theorem}\label{theo:node-treewidth}
  For all graphs $G,H$ and all $v\in V(G),w\in V(H)$, the following are
  equivalent.
  \begin{enumerate}
  \item $\Hom_{\CT^*}(G,v)=\Hom_{\CT^*}(H,w)$ for the class $\CT^*$ of
    all rooted trees;
  \item $1$-WL assigns the same colour to $v$ and $w$.
  \end{enumerate}
\end{theorem}

This result is implicit in the proof of Theorem~\ref{theo:treewidth}
(see \cite{dvo10,delgrorat18}). In fact, it can be viewed as the graph
theoretic core of the proof. We sketch the proof here and also show
how to derive Theorem~\ref{theo:treewidth} (for trees) from it.

\begin{proof}[Proof sketch]
  Recall from Section~\ref{sec:wl-kernels} that we can view the
  colours assigned by $1$-WL as rooted trees (see
  Figure~\ref{fig:colours}). For the $k$th round of WL, this is a tree
  of height $k$, and for the stable colouring we can view it as an
  infinite tree. Suppose now that the colour of a vertex $v$ of $G$
  is $T$. The crucial observation is that for every
  rooted tree $(S,r)$ the number $\hom(S,G;r\mapsto v)$ is precisely
  the number of mappings $h:V(S)\to V(T)$ that map the root $r$ of $S$
  to the
  root of $T$ and, for each node $s\in V(S)$, map the children of $s$
  to the children of $h(s)$ in $T$. Let us call such mappings
  \emph{rooted tree homomorphisms}.

  Implication (2)$\implies$(1) follows directly from this
  observation. Implication (1)$\implies$(2) follows as well, because
  by an argument similar to that used in the proof of
  Theorem~\ref{theo:lov} it can be shown that for distinct rooted
  trees $T,T'$ there is a rooted tree that has distinct numbers of
  rooted tree homomorphisms to $T,T'$.
\end{proof}

\begin{corollary}\label{cor:home-ne}
  For all graphs $G,H$ and all $v\in V(G),w\in V(H)$, the following are
  equivalent.
  \begin{enumerate}
  \item $\Hom_{\CT^*}(G,v)=\Hom_{\CT^*}(H,w)$;
  \item for all formulas $\phi(x)$ of the logic $\LC^2$,
    \[
      G\models\phi(v)\iff H\models\phi(w).
    \]
  \end{enumerate}
\end{corollary}

Note that the node embeddings based on homomorphism vectors are quite
different from the node embeddings described in
Section~\ref{sec:node-emb}. They are solely based on structural properties
and ignore the distance information. Results like
Corollary~\ref{cor:home-ne} show that the structural information
captured by the homomorphism-based embeddings in principle enables us to answer
queries directly on the embedding, which may be more useful than
distance information in database applications.

We close this section by sketching how Theorem~\ref{theo:treewidth}
(for trees) follows from Theorem~\ref{theo:node-treewidth}.

\begin{proof}[Proof sketch of Theorem~\ref{theo:treewidth} (for trees)]
  Let $G$, $H$ be graphs. We 
need to prove 
  \begin{equation}
    \label{eq:7}
     \Hom_{\CT}(G)=\Hom_{\CT}(H)\iff \text{$1$-WL does not
      distinguish $G$ and $H$.}
   \end{equation}
  Without loss of generality we assume that
  $V(G)\cap V(H)=\emptyset$. For $x,y\in V(G)\cup V(H)$, we write
  $x\sim y$ if $1$-WL assigns the same colour to $x$ and
  $y$. By Theorem~\ref{theo:node-treewidth}, for
  $X,Y\in\{G,H\}$ and $x\in V(X)$, $y\in V(Y)$ we have $x\sim y$ if
  and only if
  $\Hom_{\CT^*}(X,x)=\Hom_{\CT^*}(X,y)$. Let $R_1,\ldots,R_n$ be the
$\sim$-equivalence classes, and for every $j$, let $P_j:=R_j\cap V(G)$
and $Q_j:=R_j\cap V(H)$. Furthermore, let $p_j:=|P_j|$ and
$q_j:=|Q_j|$.

We first prove the backward direction of \eqref{eq:7}. Assume that
$1$-WL does not distinguish $G$ and $H$. Then $p_j=q_j$ for all
$j\in[n]$. Let $T$ be a tree, and let $t\in V(T)$.
Let $h_j:=\hom(T,X;t\mapsto x)$ for $x\in R_j$ and $X\in\{G,H\}$
with $x\in V(X)$. Then
\begin{align*}
  \hom(T,G)&=\sum_{v\in V(G)}\hom(T,G;t\mapsto v)\\
           &=\sum_{j=1}^np_jh_j=\sum_{j=1}^nq_jh_j\\
           &=\sum_{w\in V(H)}\hom(T,H;t\mapsto w)=\hom(T,H).
\end{align*}
Since $T$ was arbitrary, this proves
$\Hom_{\CT}(G)=\Hom_{\CT}(H)$.

The proof of the forward direction of \eqref{eq:7} is more
complicated. Assume $\Hom_{\CT}(G)=\Hom_{\CT}(H)$.
  There is
  a finite collection
  of $m\le \binom{n}{2}$ rooted trees $(T_1,r_1),\ldots,(T_m,r_m)$ such that for all
  $X,Y\in\{G,H\}$ and $x\in V(X)$, $y\in V(Y)$ we have $x\sim y$ if
  and only if for all $i\in[m]$,
  \[
    \hom(T_i,X;r_i\mapsto x)=\hom(T_i,Y;r_i\mapsto y).
  \]
  Let $a_{ij}:=\hom(T_i,X;r_i\mapsto x)$ for $x\in R_j$ and $X\in\{G,H\}$
  with $x\in V(X)$. Then for all $i$ we have
  \begin{equation*}
    \label{eq:8}
    \sum_{j=1}^na_{ij}p_j=\hom(T_i,G)=\hom(T_i,H)=\sum_{j=1}^na_{ij}q_j
  \end{equation*}
  Unfortunately, the matrix $A=(a_{ij})_{i\in[m],j\in[n]}$ is not
  necessarily invertible, so we cannot directly conclude that
  $p_j=q_j$ for all $j$. All we know is that for any two columns of
  the matrix there is a row such that the two columns have distinct
  values in that row. It turns out that this is sufficient. For every
  vector $\vec d=(d_1,\ldots,d_m)$ of nonnegative integers, let
  $(T^{(\vec d)},r^{(\vec d)})$ be the rooted tree obtained by taking
  the disjoint union of $d_i$ copies of $T_i$ for all $i$ and then
  identifying the roots of all these trees. It is easy to see that
    \[
      \hom(T^{(\vec d)},X;r^{(\vec d)}\mapsto
      x)=\prod_{i=1}^m\hom(T_i,X;r_i\mapsto x).
    \]
    Thus, letting $a_j^{(\vec d)}:=\prod_{i=1}^ma_{ij}^{d_i}$, we have
    \begin{equation*}
      \label{eq:9}
      \sum_{j=1}^na_j^{(\vec d)}p_j=\hom(T^{(d)},G)=\hom(T^{(d)},H)=\sum_{j=1}^na_j^{(d)}q_j.
    \end{equation*}
    Using these additional equations, it can be shown that $p_j=q_j$
    for all $j$ (see \cite[Lemma~4.2]{gro20a}). Thus WL does not
    distinguish $G$ and $H$.
\end{proof}

\subsection{Homomorphisms and GNNs}
We have a correspondence between homomorphism vectors and and the
Weisfeiler-Leman algorithm (Theorems~\ref{theo:treewidth} and
\ref{theo:node-treewidth}) and between the the WL algorithm and GNNs
(see Section~\ref{sec:wl-gnn}). This also establishes a correspondence
between homomorphism vectors and GNNs. More directly, the correspondence
between GNNs and homomorphism counts is also studied in \cite{maehoa19}.

\section{Similarity}
\label{sec:sim}
The results described in the previous section can be interpreted as
results on the expressiveness of homomorphism-based embeddings of
structures and their nodes. However, all these results only show what
it means that two objects are mapped to the same homomorphism
vector. More interesting is the similarity measure the
vector embeddings induce via some an inner product/norm on
the latent space (see \eqref{eq:30}). We can
speculate that, given the nice results regarding equality of vectors,
the similarity measure will have similarly nice properties.
Let me propose the following, admittedly vague, hypothesis.
\begin{quote}
  \itshape
  For suitable
classes $\CF$, the homomorphism embedding
$\Hom_{\CF}$ combined with a suitable inner product on the latent space
induces a natural similarity measure on graphs or relational structures.
\end{quote}
From a practical perspective, we could support this hypothesis by showing that the
vector embeddings give good results when combined with similarity
based downstream tasks. As mentioned earlier, initial experiments show that homomorphism
vectors in combination with support vector machines perform well on
standard graph classification benchmarks. But
a more thorough experimental study will be required to have conclusive
results.

From a theoretical perspective, we can compare the homomor\-phism-based
similarity measures with other similarity measures for graphs and
discrete structures. If we can prove that they
coincide or are close to each other, then this would support our hypothesis.

\subsection{Similarity from Matrix Norms}
A standard way of defining similarity measures on graphs is based on
comparing their adjacency matrices. Let us briefly review a few matrix
norms. Recall the standard $\ell_p$-vector norm $\|\vec
x\|_p:=\left(\sum_i|x_i|^p\right)^{1/p}$. \footnote{Note that
  $\|\vec x\|_2$ is just the
Euclidean norm, which we denoted by $\|\vec x\|$ earlier in this
paper.} 
The two best-known matrix norms are the \emph{Frobenius norm}
$\|M\|_F:=\sqrt{\sum_{i,j}M_{ij}^2}$ and the \emph{spectral norm}
$\|M\|_{\angles2}:=\sup_{\vec x\in\Real^n,\|\vec x\|_2=1}\|M\vec
x\|_2$. More generally, for every $p>0$ we define
\begin{align*}
\|M\|_p&:=\left(\sum_{i,j}|M_{ij}|^p\right)^{1/p}\\
\intertext{(so
$\|M\|_F=\|M\|_2$) and}
\|M\|_{\angles p}&:=\sup_{\vec x\in\Real^n,\|\vec x\|_p=1}\|M\vec
x\|_p,.
\end{align*} 
Thus both $\|M\|_p$ and  $\|M\|_{\angles p}$ are derived from
the $\ell_p$-vector norm. For $\|M\|_p$, the matrix is simply
flattened into a vector, and for $\|M\|_{\angles p}$ the matrix is
viewed as a linear operator. Another matrix norm that is interesting here is the \emph{cut
  norm} defined by
\[
\|M\|_{\square}:=\max_{S,T}\Big|\sum_{i\in S,j\in T}M_{ij}\Big|,
\]
where $S, T$ range over all subsets of the index set of the
matrix. Observe that for $M\in\Real^{n\times n}$ we have
\[
  \|M\|_{\square}\le \|M\|_1\le n\|M\|_F,
\]
where the second inequality follows from the Cauchy-Schwarz
inequality.  If we compare matrices of different size, it can be
reasonable to scale the norms by a factor depending on $n$. 

For
technical reasons, we only consider matrix norms $\|\cdot\|$ that are invariant under
permutations of the rows and columns, that is,
\begin{equation}
  \label{eq:11}
  \|M\|=\|MP\|=\|QM\|\quad\text{for all permutation matrices }P,Q.
\end{equation}
It is easy to see that the norms discussed above have this property.

Now let $G,H$ be graphs with vertex
sets $V,W$ and adjacency matrices $A\in\Real^{V\times V},
B\in\Real^{W\times W}$. For convenience, let us assume that
$|G|=|H|=:n$.
Then both $A,B$ are $n\times n$-matrices, and we can compare them
using a matrix norm. However, it does not make much
sense to just consider $\|A-B\|$, because graphs do not have a unique
adjacency matrix, and even if $G$ and $H$ are isomorphic, $\|A-B\|$
may be large. Therefore, we align the two matrices in an optimal way
by permuting the rows and columns of $A$.  For a matrix norm $\|\cdot\|$, we define
a graph distance measure $\dist_{\|\cdot\|}$
\[
  \dist_{\|\cdot\|}(G,H):=\min_{\substack{P\in\{0,1\}^{V\times
        W}\\\text{permutation matrix}}}\|P^{\top}AP-B\|.
\]
It follows from \eqref{eq:11} that 
$\dist_{\|\cdot\|}$ is well-defined, that is, does not depend on the
choice of the particular adjacency matrices $A,B$. It also follows
from \eqref{eq:11} and that fact that $P^{-1}=P^\top$ for permutation
matrices that
\begin{equation}
  \label{eq:12}
  \dist_{\|\cdot\|}(G,H)=\min_{\substack{P\in\{0,1\}^{V\times
        W}\\\text{permutation matrix}}}\|AP-PB\|,
\end{equation}
which is often easier to work with because the expression  $AP-PB$ is
linear in the ``variables'' $P_{ij}$. To simplify the notation, we let
$\dist_{p}:=\dist_{\|\cdot\|_p}$ and $\dist_{\angles
  p}:=\dist_{\|\cdot\|_{\angles p}}$ for all $p$, and we let
$\dist_{\square}:=\dist_{\|\cdot\|_\square}$.

The distances defined from the $\ell_1$-norm have natural
interpretations as edit distances. 
$\dist_1(G,H)$ is twice the number of edges that need to be
flipped to turn $G$ into a graph isomorphic to $H$, and
$\dist_{\angles1}(G,H)$ is the maximum number of edges incident with a
single vertex we need to flip to turn $G$ into a graph isomorphic to
$H$. Formally,
\begin{align}
\label{eq:40}
&\dist_{1}(G,H)=2\min_{\substack{f:V\to W\\
    \text{bijection}}}\Big|\big\{f(v)f(v')\bigmid vv'\in
  E(G)\big\}\triangle E(H)\Big|.\\
\notag
&\dist_{\angles1}(G,H)=\\
\label{eq:41}
&\hspace{5mm}\min_{\substack{f:V\to W\\
    \text{bijection}}}\max_{v\in
  V}\Big|\big\{f(v')\bigmid v'\in N_G(v)\big\}\triangle 
\big\{w'\bigmid w\in N_H(f(v))\big\}\Big|.
\end{align}
Here $\triangle$ denotes the symmetric difference. Equation
\eqref{eq:40} is obvious; the factor '$2$'
comes from the fact that we regard (undirected) edges as 2-element
sets and not as ordered pairs. To prove \eqref{eq:41}, we observe that for every matrix
$M\in\Real^{n\times n}$ it holds that
$
  \|M\|_{\angles1}=\max_{j\in[n]}\sum_{i=1}^n|M_{ij}|.
$

Despite these intuitive
interpretations, it is debatable how much ``semantic relevance''
these distance measures have. How similar are two graphs that can be
transformed into each other by flipping, say, 5\% of the edges? Again,
the answer to this question may depend on the application context.

A big disadvantage the graph distance measures based on matrix norms
have is that computationally they are highly intractable (see, for
example, \cite{arvkobkuhvas12} and the references therein). It is even
NP-hard to compute the distance between two trees (see
\cite{groratwoe18} for Frobenius distance and \cite{ger18} for the
distances based on operator norms), and distances are hard to
approximate. The problem of computing these distances is related to
the maximisation version of the quadratic assignment problem (see
\cite{makmansvi14,nagsvi09}), a notoriously hard combinatorial
optimisation problem. Better behaved is the cut-distance $\dist_{\square}$; at
least it can be approximated within a factor of $2$ \cite{alonao06}.

The main source of hardness is the minimisation over the
unwieldy set of all permutations (or permutation matrices). To
alleviate this hardness, we can relax the integrality constraints and
minimise over the convex set of all doubly stochastic matrices
instead. That is, we define a relaxed distance measure
\begin{equation}
  \label{eq:12}
  \widetilde{\dist}_{\|\cdot\|}(G,H)=\min_{\substack{X\in\{0,1\}^{V\times
        W}\\\text{doubly stochastic}}}\|AX-XB\|.
\end{equation}
Note that $\widetilde{\dist}_{\|\cdot\|}$ is only a pseudo-metric: the
distance between nonisomorphic graphs may be $0$. Indeed, it follows
from Theorem~\ref{theo:fractional} that $\widetilde{\dist}_{\|\cdot\|}(G,H)=0$ if and only if $G$ and $H$ are fractionally isomorphic. The
advantage of these ``relaxed'' distances is that for many
norms $\|\cdot\|$, computing $\widetilde{\dist}_{\|\cdot\|}$ is a
convex minimisation problem that can be solved efficiently.

So far, we have only discussed distance measures for graphs of the
same order. To extend these distance measures to arbitrary graphs, we
can replace vertices by sets of identical vertices in both graphs to
obtain two graphs whose order is the least common multiple of the
orders of the two initial graphs (see \cite[Section~8.1]{lov12} for
details).

Note that these matrix based similarity measures are only defined for
(possibly weighted) graphs. In particular for the operator norms, it
is not clear how to generalise them to relational structures, and if
such a generalisation would even be meaningful.

\subsection{Comparing Homomorphism Distances and Matrix Distances}

It would be very nice if we could establish a connection between graph
distance measures based on homomorphism vectors and those based on
matrix norms. At least one important result in this direction exists:
Lov{\'a}sz \cite{lov12} proves an equivalence between the cut-distance of
graphs and a distance measure derived from a suitably scaled
homomorphism vector $\Hom_{\CG}$.

It is tempting to ask if a similar correspondence can be established
between $\widetilde{\dist}_{\|\cdot\|}$ and $\Hom_{\CF}$. There are
many related question that deserve further attention.

\section{Concluding Remarks}
In this paper, we gave an overview of embeddings techniques for graphs
and relational structures. Then we discussed two related theoretical
approaches, the Weisfeiler-Leman algorithm with its various
ramifications and homomorphism vectors.
We saw that they
have a rich and beautiful theory that leads to new, generic
families of vector embeddings and helps us to get a better
understanding of some of the techniques used in practice, for example
graph neural networks.

Yet we have also seen that we are only at the beginning and many
questions remain open, in particular when it comes to similarity
measures defined on graphs and relational structures.

From a database perspective, it will be important to generalise the
embedding techniques to relations of higher arities, which is not as
trivial as it may seem (and where surprisingly little has been done so
far).  A central question is then how to query the embedded
data. Which queries can we answer at all when we only see the vectors
in latent space? How do imprecisions and variations due to randomness
affect the outcome of such query answers? Probably, we can only answer
queries approximately, but what exactly is the semantics of such
approximations? These are just a few questions that need to be
answered, and I believe they offer very exciting research opportunities
for both theoreticians and practitioners.

\subsection*{Acknowledgements}
This paper was written in strange times during COVID-19 lockdown. I
appreciate it that some of my colleagues nevertheless took the time to
answer various questions I had on the topics covered here and to
give valuable feedback on an earlier version of this paper. In
particular, I would like to thank Pablo Barcel\'o, Neta Friedman,
Benny Kimelfeld,
Christopher Morris, Petra Mutzel, Martin Ritzert, and
Yufei Tao.

\end{document}